\documentclass{article}

\usepackage[main, final]{neurips_2026}
\makeatletter
\g@addto@macro\@bottomtitlebar{\vskip -0.4cm}
\makeatother

\usepackage[utf8]{inputenc} 
\usepackage[T1]{fontenc}    
\usepackage{hyperref}       
\usepackage{url}            
\usepackage{booktabs}       
\usepackage{amsfonts}       
\usepackage{nicefrac}       
\usepackage{microtype}      
\usepackage{xcolor}         

\usepackage{amsmath}
\usepackage{graphicx}
\usepackage{multirow}
\usepackage[table]{xcolor}
\usepackage{graphicx}
\usepackage{multirow}
\usepackage{amsmath}
\usepackage{amsthm}
\usepackage[table]{xcolor}
\usepackage{arydshln}
\usepackage{soul} 
\usepackage[dvipsnames]{xcolor}
\usepackage[dvipsnames, table]{xcolor}
\usepackage{graphicx}
\usepackage{subcaption}
\usepackage{lipsum}
\usepackage{multirow}
\usepackage{array}
\usepackage{xspace}
\usepackage[most]{tcolorbox}
\usepackage{enumitem}
\usepackage{wrapfig}

\usepackage[capitalise]{cleveref}
\usepackage{makecell}

\usepackage{mathtools}
\DeclarePairedDelimiter{\norm}{\lVert}{\rVert}

\newtheorem{theorem}{Theorem}

\title{CSFlow: Aligning Flow Matching with 

Human Contrast Sensitivity\vspace{-0.08cm}}
\vspace{-1.4cm}

\vspace{-1.4cm}
{
\author{
Malgorzata Galinska$^{1}$\thanks{Equal contribution.} \quad
Bart Pogodzinski$^{1}$\footnotemark[1] \quad
Jan Eric Lenssen$^{1,2}$
\\[6pt]
{\normalsize
$^{1}$Max Planck Institute for Informatics
}
\\
{\normalsize
$^{2}$Saarland University,
Saarland Informatics Campus, Germany
}
\\[5pt]
{\small
\texttt{\{mgalinsk, bpogodzi, jlenssen\}@mpi-inf.mpg.de}
}}
}

\begin{document}

\maketitle

{

\begin{figure}[h]
\centering
\vspace{-20pt}

\includegraphics[width=1.0\linewidth]{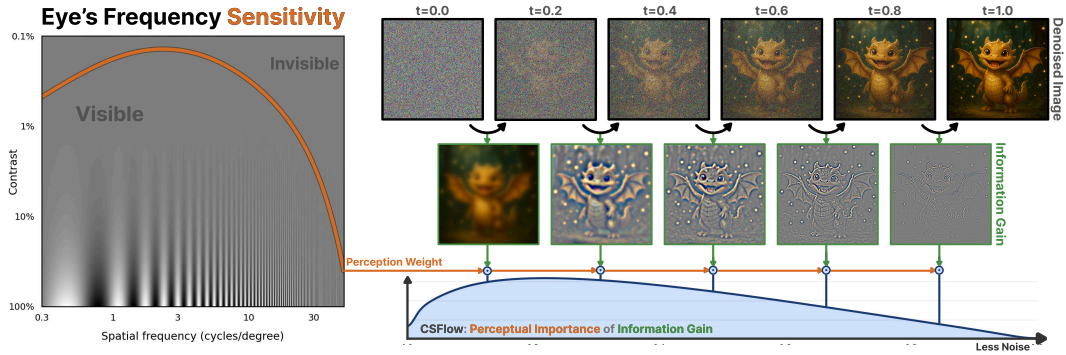}

\caption{Contrast Sensitive Flow (CSFlow) connects the human eye's Contrast Sensitivity Function (CSF)~\citep{csf} (left) to the iterative denoising steps of flow matching (right), providing a weighting scheme grounded in human vision. Information gain in Fourier space of each denoising step is weighted by the perceptual importance of generated frequencies, allocating more capacity to frequencies to which the human eye is more sensitive.}
\vspace{-0.2cm}

\label{fig:teaser}
\label{fig:CSF-barten} 
\end{figure}

}

\begin{abstract}
\vspace{-0.1cm}

We introduce Contrast Sensitive Flow (CSFlow), a weighting scheme that connects the human eye's Contrast Sensitivity Function (CSF) to the iterative denoising steps of flow matching. 
Because real-world images concentrate signal at low spatial frequencies, these components reach high signal-to-noise ratio earlier during continuous diffusion than high-frequency components. When generating images with diffusion or flow matching models, this induces a soft autoregressive structure in Fourier space, where coarse image content stabilizes before fine detail. 
Meanwhile, the human visual system is unequally sensitive to spatial frequencies: very low and very high frequencies require significantly higher contrast to be perceived. 
We for the first time
merge these observations through two contributions: (1) a metric that estimates which frequencies are generated at each reverse flow interval and (2) timestep weights obtained by aligning the frequencies generated at each noise level with human contrast sensitivity. We validate our contributions experimentally showing that these weights can improve generative performance by lowering FID by up to $12.15\%$, increasing Inception Score by up to $3.2\%$ and improving GenEval scores by up to $3.0\%$ using inference-only timestep modification or short fine-tuning.
Qualitatively, we find that our CSFlow weights lead to better visual realism and less cartoonish appearance of generated images.
In a two-alternative forced-choice study, participants preferred CSFlow generations over the baseline in $65\%$ of comparisons.
Additional visualisations, code, models and the example CSFlow weights are available on our project website\footnote{\url{gosiagal.github.io/CSFlow}}.

\end{abstract}

\section{Introduction}
Diffusion~\citep{diffusion1, ddpm} and flow matching~\citep{lipman2022flow, liu2022flow} models have recently become some of the most successful paradigms for image synthesis, enabling the generation of diverse and high-quality images~\citep{pixelgen}. Although model architectures, training objectives, etc., can differ, 
one of the \emph{most common} objectives is to produce high-quality images that are viewed and judged by humans, with some limited training and inference compute budget.

Research on the human eye shows that there exists a \emph{window of well-noticable spatial frequencies} for human perception (c.f. Fig~\ref{fig:CSF-barten}, left)~\citep{csf}. 
Certain image variations, such as changes in very low or high frequency components (e.g., very coarse gradients or texture details), are largely invisible to the human eye, even if the contrast is high. Meanwhile, other frequencies, mostly located in intermediate frequency bands, are easy to spot even with low contrast. Consequently, unnatural representation of those low-mid-range frequencies is the most noticeable to the human eye. 

Recent works in generative modeling with diffusion and flow models strongly connect the different stages of the reverse process to the reconstruction of different spatial frequencies \citep{spectral, fourier}. Over the course of the reverse process, images are generated in a coarse-to-fine manner, where information is added in low frequencies first before adding high frequency details towards the end. In a setting with a limited compute budget and model capacity, this poses the question of whether an image generation model should allocate the budget to reconstruct frequencies that are invisible to us humans, or if it should rather prioritize allocating its capabilities to the most visible ones. With this work, we are the first to connect these two phenomena and show that weighting the information gain in reverse flow matching steps to favor steps that are important for human perception (c.f. Fig.~\ref{fig:CSF-barten}, right) can improve the quality of generated images.

Concretely, we propose Contrast Sensitive Flow (CSFlow), a procedure to calculate \emph{perception-driven weights} for flow matching. The weighting can be applied in two ways: (1) to modify the step sizes during inference, increasing the approximation accuracy of important intervals of the generation process; (2) to weight the probability of sampling certain $t$ during model training, favoring the allocation of more model capacity to specific generation intervals. To this end, we describe a metric to estimate which frequencies are generated in which parts of the flow matching reverse process and how contrast sensitive weights can subsequently be obtained from the resulting connection.

We conduct experiments on representatives of current state-of-the-art pixel-space flow matching models, i.e., PixelGen~\citep{pixelgen}, PixelREPA~\citep{pixelrepa} and JiT~\citep{jit}, evaluating the GenEval~\citep{geneval}, FID, and Inception score, which represent visual quality obtained by applying our training weights and inference step sizes. Our findings suggest that perception-driven weighting during training and inference consistently improves the visual quality of generation by up to 12.15\% in FID, 3.2\% in Inception Score, and 3.0\% in GenEval scores. Qualitatively, we find that our generated images appear to be more realistic and less \emph{cartoonish} than those from the baselines, and a two-alternative forced-choice study confirms that human viewers prefer them in 65\% of cases.
Our concrete contributions are the following: 
\begin{itemize}[leftmargin=1.2em,itemsep=0.8pt,topsep=1pt]
    \item A closed-form metric $\partial_tr_\mathrm{signal}(f,t)$ 
    that estimates which spatial frequencies are \emph{being generated} in each interval of the reverse flow. To our knowledge, we are the first to describe explicitly which frequency content is generated at a specific $t$.
    \item \textbf{CSFlow}: perception-driven timestep weights, obtained by aligning the frequencies generated at each noise level with human contrast sensitivity.
    \item Two drop-in applications of these weights: as as a non-uniform \textbf{inference} step grid and a \textbf{training} timestep distribution. Neither of them requires new training losses or architecture changes.
\end{itemize}

\section{Related Work}
Diffusion models can change the relative importance of denoising stages through the noise schedule, distribution of timesteps ($t$), explicit loss weights, and the choice of the prediction target \mbox{($\epsilon$, $x_{clean}$, $v$)}. Those have been implemented across many works with varying backgrounds and motivation.

DDPM~\citep{ddpm} derives timestep-dependent loss terms from the ELBO, but obtains better sample quality with a simplified denoising objective that drops the ELBO-derived weighting. Improved DDPM~\citep{iddpm} introduces a cosine noise schedule, which changes how signal is retained across the timesteps and makes the time discretization leading to improved results.
EDM~\citep{edm} shows that the treatment of different noise levels is shaped by separable design choices, including the training noise-level distribution, the denoising loss weight, network parametrization, and the sampling-time discretization.
Min-SNR-$\gamma$~\citep{minsnrgamma} treats the diffusion as a multi-task-learning problem and notices that training convergence and final model quality can improve with loss weighted according to the clamped signal-to-noise ratios, preventing the high-loss (and thus high gradient) regions from overpowering the optimization -- resulting in downweighing high-noise $t$ in $x_{clean}$-prediction and downweighing the low-noise regions in $\epsilon$-prediction.
EDiff-I~\citep{ediffi} takes a different approach by training an ensemble of denoisers for individual noise intervals, essentially mitigating the issue some noise levels overpowering others during training, at the cost of additional training and system complexity.

High-resolution and flow models further modify timestep importance through schedule shifts and non-uniform sampling. Simple Diffusion~\citep{sid} argues that increasing image resolution changes the effective SNR of low-frequency structure: averaging over more pixels reduces the noise variance of coarse image content. They therefore shift the noise schedule for high-resolution pixel-space diffusion so that the \textit{effective} SNR at a reference resolution is preserved.
Stable Diffusion 3~\citep{sd3} applies a related resolution-dependent timestep shift in latent-space rectified flows, selecting the shift through inference-time sampling experiments and then using it during both high-resolution training and inference. It also studies non-uniform training-time timestep sampling, including logit-normal sampling, which emphasizes intermediate timesteps and vanishes near the endpoints.
Schedule On the Fly~\citep{ye2025schedule} targets inference-time compute allocation by adding a next noise level predictor via reinforcement learning. 
All of the above methods aim to improve the understanding and learnability of diffusion models but do not align it with human perception.

\begin{wrapfigure}{r}{0.595\textwidth}
\vspace{-0.4cm}
\centering
    \begin{minipage}{0.48\linewidth}
        \centering
        \includegraphics[width=\linewidth]{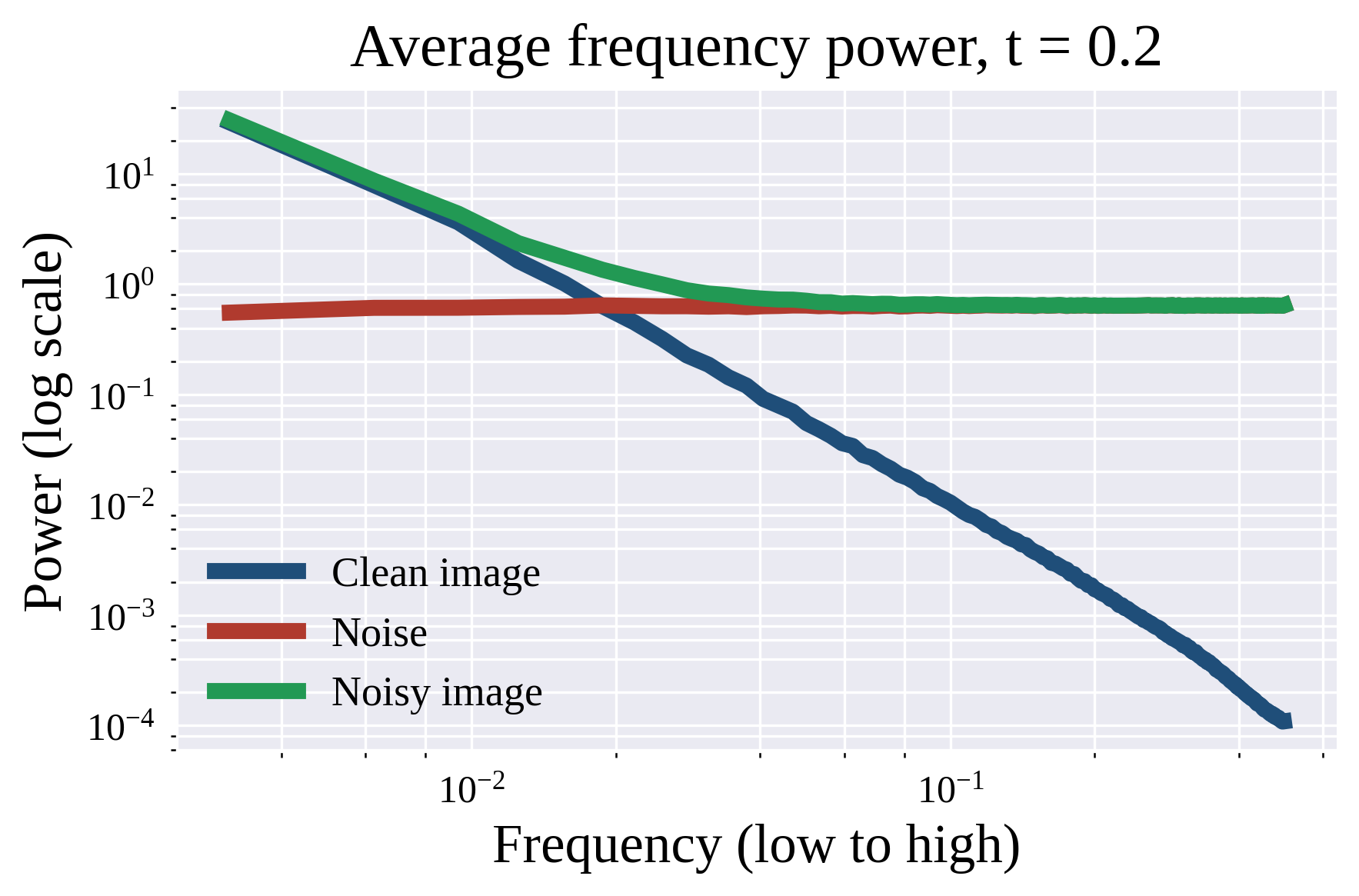}
    \end{minipage}
    \hfill
    \begin{minipage}{0.48\linewidth}
        \centering
        \includegraphics[width=\linewidth]{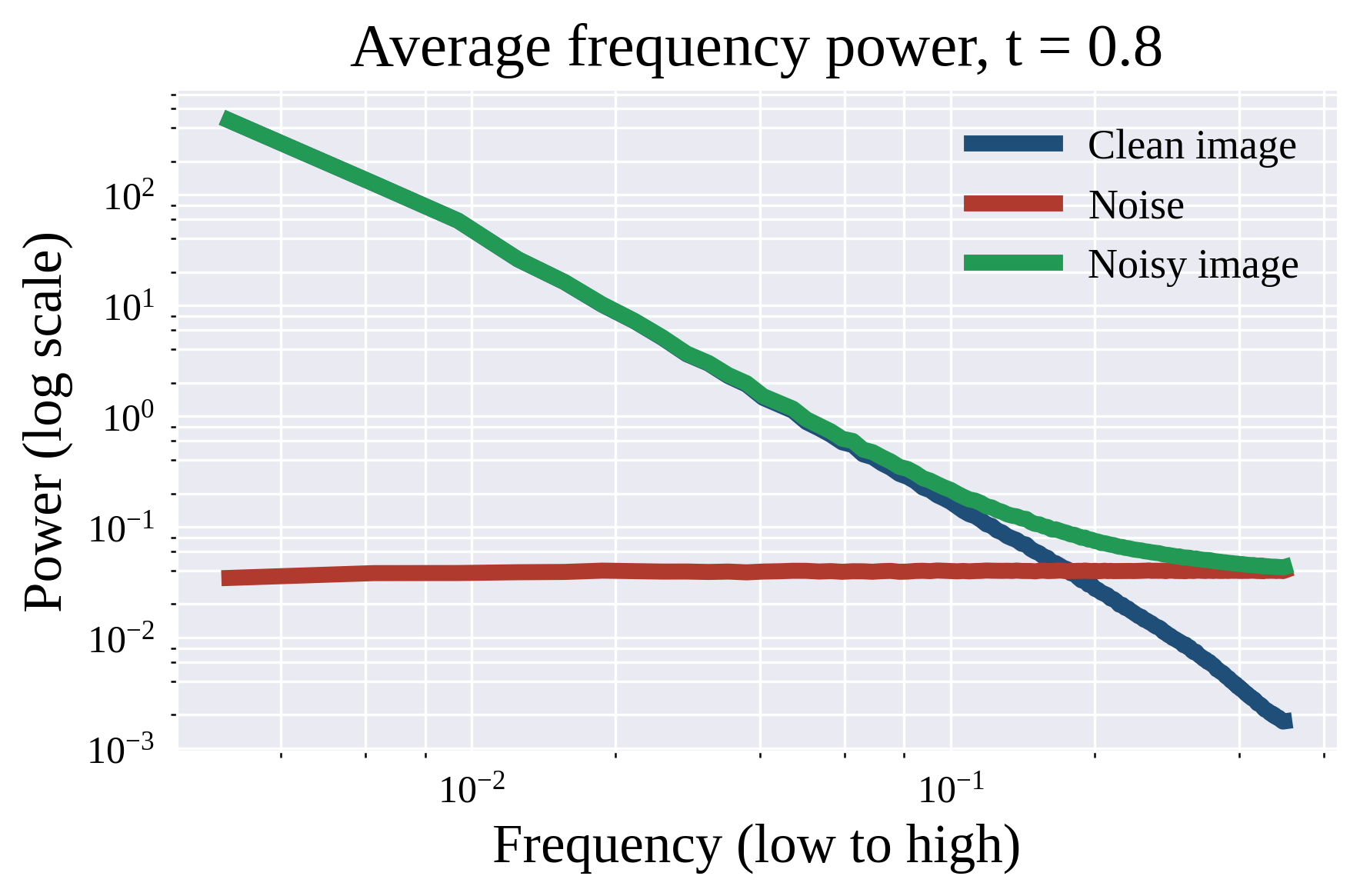}
    \end{minipage}
    \caption{Comparison of average frequency strengths between two denoising stages: \(t=0.2\) -- early stage, mostly noise and \(t=0.8\) -- late stage, mostly data. Natural image power spectrum follows a power law and noise power spectrum is constant. The noisy image magnitude (green) aligns with the clean image (blue) and noise magnitude (red) on different ranges depending on the \(t\). The later the stage, the higher frequency bands are reconstructed in the noisy image.}
\vspace{-1cm}
    \label{fig:rapsd}
\end{wrapfigure}

P2~\citep{p2} aims to align the generation task to human perception by observing that late denoising stages mostly correspond to noise-cleanup phase, where no visible-to-humans generative decisions are left to be made. They heuristically 
propose loss weights, which are lower for late denoising stages. While the approach shares the general motivation with our work, our method is grounded in a mathematical connection, deriving timestep weighting directly from the contrast sensitivity function, rather than from~coarse ~observations.

\vspace{-0.2cm}
\section{Preliminaries}
\textbf{Flow Matching.}
The flow matching model defines a continuous denoising trajectory $\mathbf{x}_t$, $t\in[0,1]$, which transforms a noise distribution at $t=0$ into the data distribution at $t=1$. In flow matching \citep{lipman2022flow}, a neural network $v_\theta(\mathbf{x}_t,t)$ is trained to predict the velocity of this trajectory, typically using an objective of the form:
\vspace{-0.1cm}
\begin{equation}
\mathcal{L}_{\mathrm{FM}}=
\mathbb{E}_{t,\mathbf{x}_0,\mathbf{x}_1}
\left[
\left|
v_\theta(\mathbf{x}_t,t)-\mathbf{u}_t
\right|^2
\right],
\end{equation}
\vspace{-0.4cm}

where $\mathbf{u}_t$ is the target velocity conditioned on $x_1$. Thus, $t$ specifies the denoising state at which the model is trained, 
and the distribution used to sample $t$ during training pressures the model to allocate more of its capacity to some noise levels.
At inference, the continuous flow is approximated numerically by a time point grid:
\begin{equation}
0=t_0<t_1<\dots<t_N=1,
\qquad
\Delta t_i=t_{i+1}-t_i,
\end{equation}
and the ODE $\frac{d\mathbf{x}_t}{dt}=v_\theta(\mathbf{x}_t,t)$ is solved over these intervals. Consequently, the distribution of training times and the discretization of inference times determine how computational and modeling capacity is allocated along the denoising trajectory.

\textbf{Frequency content of images.} The frequency domain provides a representation of images by describing them in terms of their spatial frequency components rather than their pixel intensities,
making it useful for image analysis and compression because it highlights patterns such as edges, textures, and repetitive structures that may be less obvious in pixel space. To decompose an image into spatial frequencies that create it, we use the \emph{two-dimensional Discrete Fourier Transform} (2DDFT). 

One property of natural images that is crucial to our method is that their power spectra, on average, follow a power law, with low frequencies having orders of magnitude stronger signals than the high frequencies. As discussed in \citep{spectral, fourier}, this results in imbalanced corruption of the frequency information during the diffusion forward process (c.f. Fig.~\ref{fig:rapsd}): even weak noise can destroy the highest-frequency information, but a much stronger noise is necessary to destroy the lower 
frequencies. As a result, during the generation step at some noise level $t$, only a narrow frequency band of frequencies can be considered "being generated", since lower frequencies already have too high a signal-to-noise (SNR) ratio to meaningfully change, while the higher ones will remain completely overpowered by the noise after the step. 

\textbf{Human frequency sensitivity.} The human visual system is not equally sensitive to all spatial frequencies: usually very smooth gradients (low frequencies) and tiny textural details (high frequencies) become hard to spot. An~established way to measure it is via the Contrast Sensitivity Function (CSF), which measures how much contrast is necessary for an average human eye to notice that a pattern is shown. One of the most commonly used CSF models was proposed in~\citep{csf}, which can be observed in Fig.~\ref{fig:CSF-barten}. For more details on the CSF, see Appendix \ref{csf-details}.

\section{Contrast Sensitive Flow} \label{method-chapter}
In this section, we explain how we incorporate human frequency sensitivity into the image generation of diffusion models. First, we derive \emph{perception-driven weights} that indicate the importance of every generation stage with respect to human sensitivity to the frequencies being reconstructed in these stages. We start by analyzing how the recoverable frequency content evolves over time during the denoising process in Sec.~\ref{sec:recoverable-frequency-content} and how much information is gained in each step in Sec.~\ref{sec:information_gain}. Then, we incorporate the CSF and distribute weights over time intervals in Sec.~\ref{sec:weighting_method}. In Sec.~\ref{sec:inverse-cdf}, we then explain how these weights can be used in practice to bias the model towards more important denoising stages.

\subsection{Recoverable frequency content}
\label{sec:recoverable-frequency-content}

During denoising, different spatial frequencies become distinguishable from noise at different stages. We therefore want to estimate, for each frequency \(f\) and time \(t\), how much of the spectral power present in the noisy image is attributable to the image component rather than to noise.

Let \(\mathbf{x}_{clean} \sim D\), \(\boldsymbol{\epsilon} \sim \mathcal{N}(\mathbf{0},\mathbf{I})\), and $\mathbf{x}_t = a_t \mathbf{x}_{clean} + b_t \boldsymbol{\epsilon}$, where \(a_t\) and \(b_t\) define the noise schedule. 
For a frequency \(f\), we denote the corresponding 2D DFT coefficient as \(F_f(\cdot)\) , power as $P_f(\mathbf{x}) = |F_f(\mathbf{x})|^2$, and define the expected clean image power $ S_f = \mathbb{E}\left[P_f(\mathbf{x}_{clean})\right].$

The quantity we would like to measure is the fraction of expected spectral power in \(\mathbf{x}_t\) that comes from the image component, i.e., the fraction of retained image signal at denoising state \(t\):
\begin{align}
    r_\mathrm{signal}(f,t) = 
    \frac{
        \mathbb{E}\left[P_f(a_t\mathbf{x}_{clean})\right]
    }{
        \mathbb{E}\left[P_f(\mathbf{x}_t)\right]
    }.
\end{align}

Values close to \(0\) indicate that the frequency is noise-dominated (not generated), while values close to \(1\) indicate that it is image-dominated and will not significantly change with further denoising.
The following theorem shows how to measure $r_\mathrm{signal}$ in closed-form.

\begin{theorem}[Closed form of retained signal]

\label{thm:retained-signal}

Assume that $\boldsymbol{\epsilon}\sim \mathcal{N}(0,1)$ is independent of $\mathbf{x}_{clean}$ and $S_f$, $a_t$, and $b_t$ are defined as above.

Then,
\begin{align}
    r_\mathrm{signal}(f,t) = 
    \frac{
        \mathbb{E}\left[P_f(a_t\mathbf{x}_{clean})\right]
    }{
        \mathbb{E}\left[P_f(\mathbf{x}_t)\right]
    }
    =
    \frac{
        a_t^2 S_f
    }{
        a_t^2 S_f + b_t^2
    }\text{.}
\end{align}

\end{theorem}
The proof is given in App~\ref{apn:proof-retained-signal}.
\(r_\mathrm{signal}\) can also be seen as a normalized signal-to-noise (SNR). With
\vspace{-0.15cm}
\begin{align}
    SNR_f(t)
    =
    \frac{
        P_f(a_t\cdot \mathbf{x}_{clean})
    }{
        P_f(b_t\cdot \mathbf{\epsilon})
    } = 
    \frac{
        a_t^2 S_f
    }{
        b_t^2
    },
\end{align}
\vspace{-0.1cm}
we have
\vspace{-0.15cm}
\begin{align}
    r_\mathrm{signal}(f,t)
    =
    \frac{SNR_f(t)}{1+SNR_f(t)}.
\end{align}
\vspace{-0.22cm}

Unlike raw SNR, \(r_\mathrm{signal}(f,t)\in[0,1]\) holds, making it convenient for comparing how much of a frequency should be already \emph{generated} at a specific noise level $t$. 

\subsection{Denoising step information gain from $r_\mathrm{signal}$} 
\label{sec:information_gain}

While \(r_\mathrm{signal}(f,t)\) describes how recoverable a frequency \(f\) is at a single denoising state, we are interested in which frequencies \textit{become} recoverable during each denoising interval. 
We therefore define the stepwise increase
\begin{align}
    \label{eq:delta-rsignal}
    \Delta r_\mathrm{signal}(f,t,\Delta t)
    =
    r_\mathrm{signal}(f,t+\Delta t) - r_\mathrm{signal}(f,t).
\end{align}
We use $\Delta r_\mathrm{signal}$ as a metric for the amount of \textit{newly generated} frequency content over the interval \([t,t+\Delta t]\). Thus, $\Delta r_\mathrm{signal}$ indicates the \emph{expected fraction of total frequency information we gain} at a step from $t$ to $t+\Delta t$.
Since \(r_\mathrm{signal}\) is non-decreasing in our noise schedule (and all commonly used), \(\Delta r_\mathrm{signal}(f,t,\Delta t)\geq 0\). We can also define a continuous version of $\Delta r_\mathrm{signal}$ as
$\partial_t r_\mathrm{signal}(t,f)$ 
indicating how frequencies \emph{gain information} at timestep $t$. 

Figure~\ref{fig:r-signal-three-pngs} visualizes \(r_\mathrm{signal}\) and \(\partial_t r_\mathrm{signal}\). 
Lower spatial frequencies become data-dominated earlier in the denoising process, while higher frequencies become recoverable later. 
This shows that the denoising process reconstructs data frequency content hierarchically, in a coarse-to-fine manner. We include the implementation details of estimating the weights in Appendix \ref{technical-details}.

\begin{figure}[t]
  \centering

  \begin{minipage}[b]{0.32\linewidth}
    \centering
    \includegraphics[width=\linewidth]{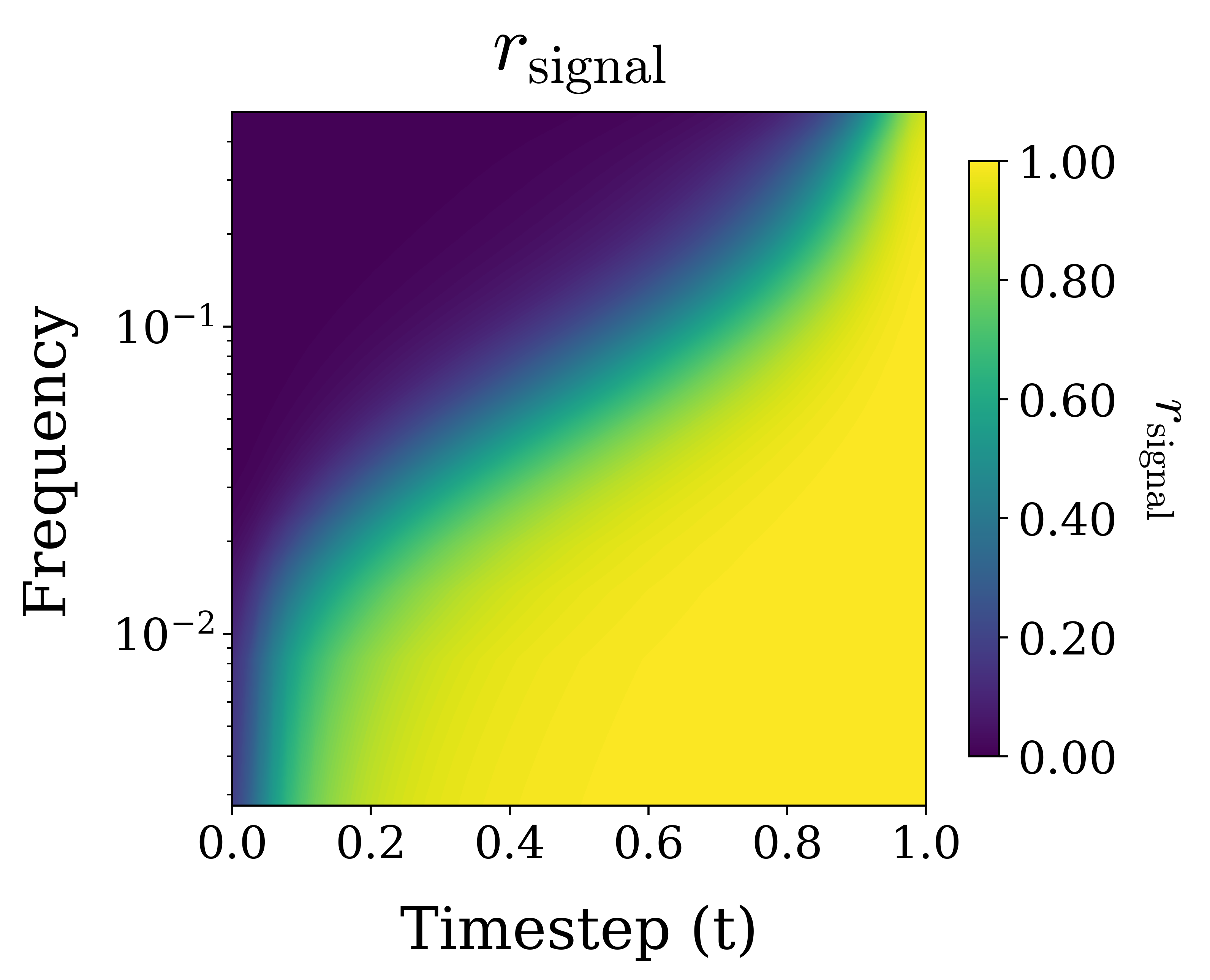}
  \end{minipage}
  \hfill
  \begin{minipage}[b]{0.32\linewidth}
    \centering
    \includegraphics[width=\linewidth]{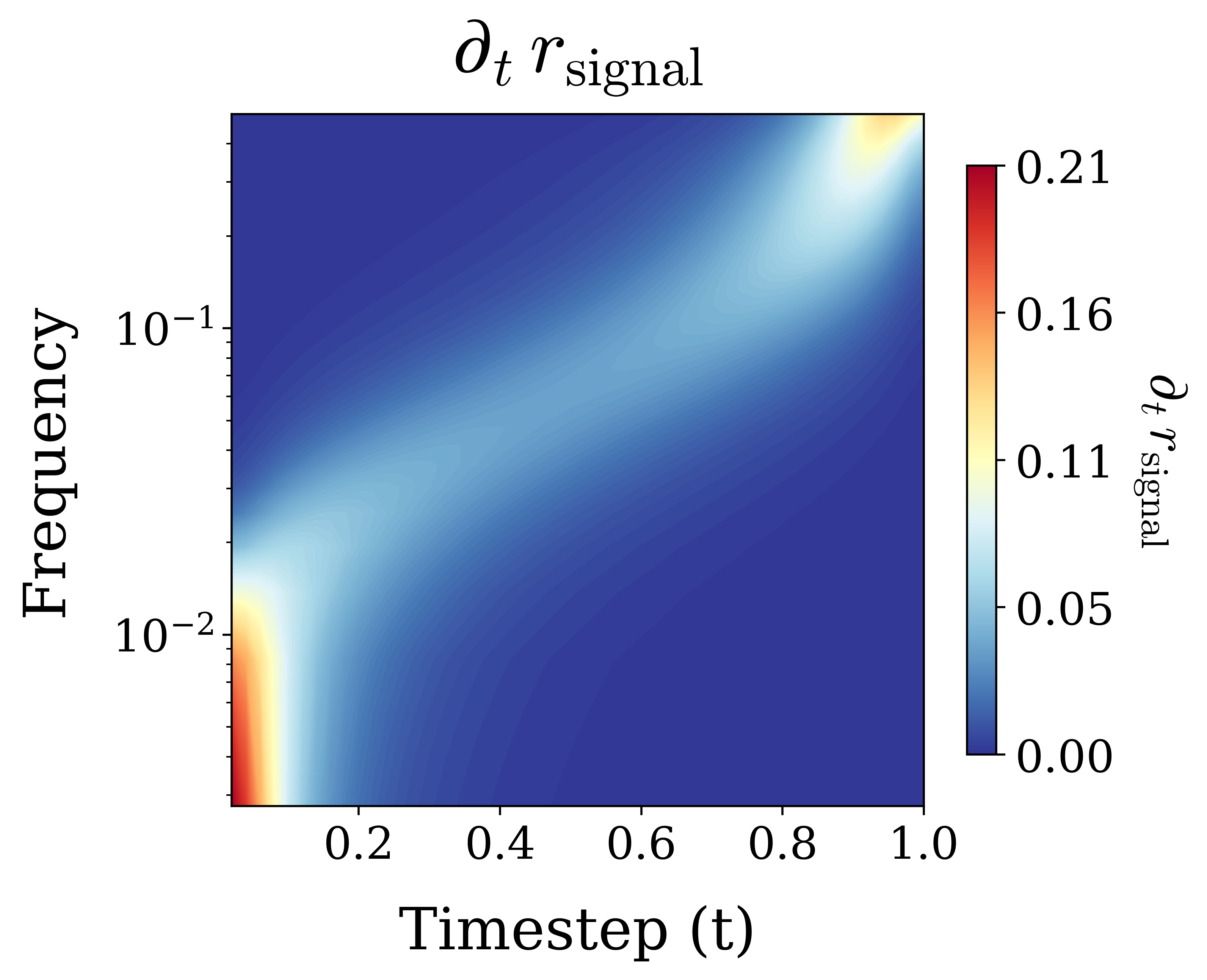}
  \end{minipage}
  \hfill
  \begin{minipage}[b]{0.32\linewidth}
    \centering
    \includegraphics[width=\linewidth]{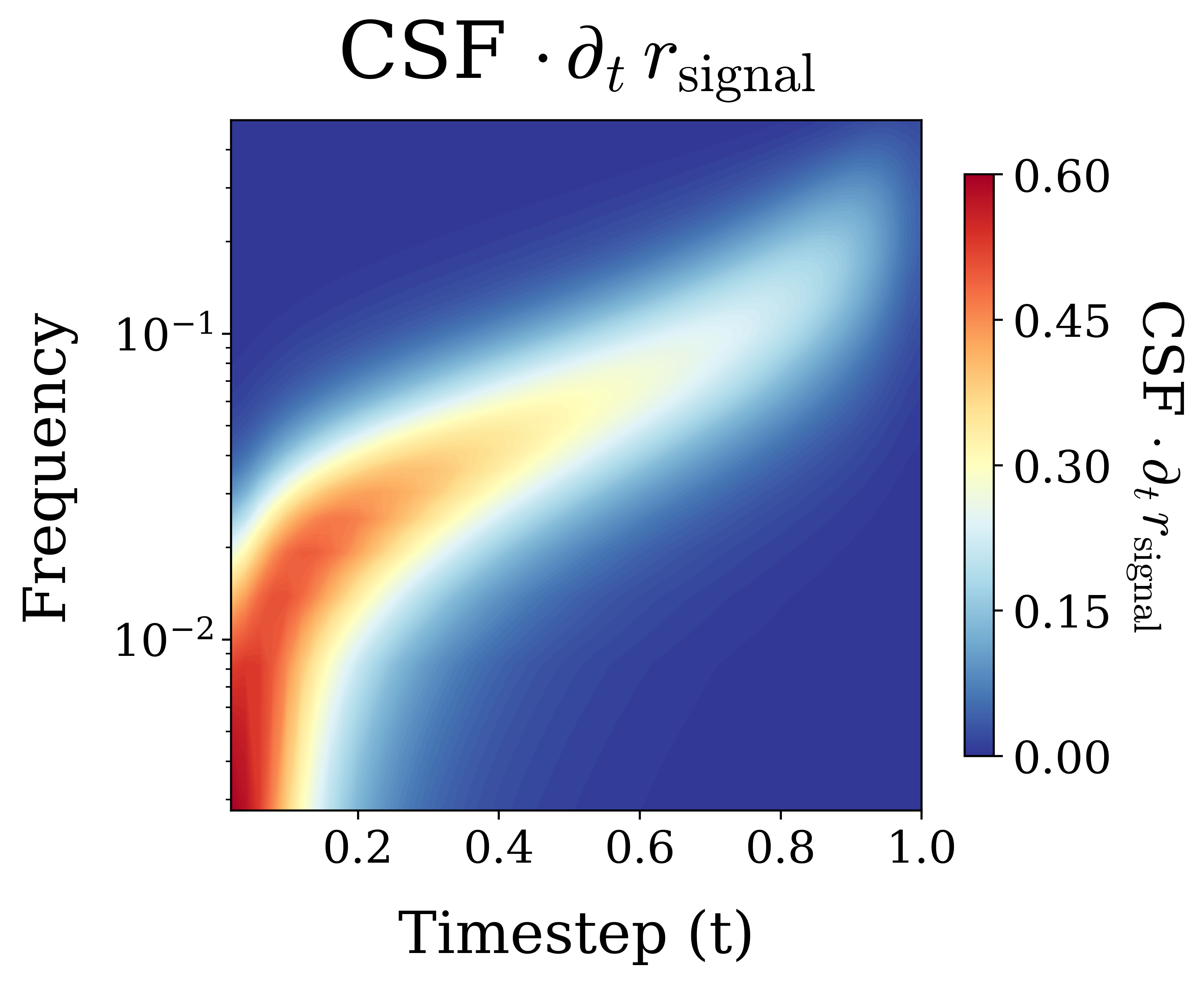}
  \end{minipage}

  \caption{$r_\mathrm{signal}$ (left), $\partial_t r_\mathrm{signal}$ (middle) and $\text{CSF} \cdot \partial_t r_\mathrm{signal}$ (right) calculated using ImageNet \(256 \times 256\) training dataset, \(50\) linear time values \(t\) and linear noise schedule \(a_t = t, \ b_t = 1-t\). Frequencies are plotted in the log space. In the plots \(t=0\) corresponds to pure noise and \(t=1\) to the clean image, therefore the denoising process happens from left to right. First two plots show that the denoising process reconstructs spatial frequencies \emph{hierarchically}, in a coarse-to-fine manner. In the last plot, high frequency values of \(\Delta r_\mathrm{signal}\) are suppressed while low and mid-range frequencies are amplified, which aligns with the CSF.}
  \vspace{-0.1cm}
  \label{fig:r-signal-three-pngs}
\end{figure}

\subsection{CSF-based interval weighting}
\label{sec:weighting_method}

Our goal is to assign larger weights to denoising intervals in which the newly recovered data-attributable frequencies are more important for human perception. 
We quantify perceptual importance using the Contrast Sensitivity Function (CSF). 
Since CSF is defined over frequencies in cycles per degree of visual angle, while our image frequencies are measured in cycles per pixel, we first convert image frequency bins to cycles per degree before evaluating CSF. 
We assume a pixel size of \(0.0114\) cm and a viewing distance of \(50\) cm. 
The conversion of units is described in Appendix~\ref{app:csf-conversion}.

Recall that $\Delta r_{\mathrm{signal}}(f,t,\Delta t)$ measures how much data-attributable content at frequency $f$ becomes newly recoverable during the interval $[t,t+\Delta t]$. In continuous notation, this corresponds to the instantaneous increase $\partial_t r_{\mathrm{signal}}(f,t)$. 
We define the \emph{raw perceptual weight} at time $t$ as the CSF-weighted average of the newly recovered frequency content:
\begin{align}
\label{eq:raw_csf_score}
    \widetilde{w}_{\mathrm{CSFlow}}(t)
    =
    \frac{
        \int_{\Omega}
        \partial_t r_{\mathrm{signal}}(f,t) \cdot CSF(f) \, df
    }{
        \int_{\Omega}
        \partial_t r_{\mathrm{signal}}(f,t) \, df
    }.
\end{align}
Here, \(\Omega\) denotes the set of frequency bins. The denominator makes the score relative: intervals are weighted according to \textit{which} frequencies are recovered, rather than simply how much total spectral content is recovered. 
Equivalently, $\widetilde{w}_{\mathrm{CSFlow}}(t)$ is the average CSF value over the frequencies that become recoverable around time $t$. 
I.e., if $CSF(f)\equiv 1$, all $t$ receive the same raw weight.

We normalize these scores over the denoising trajectory:
\begin{align}
\label{eq:csflow_weight}
    w_{\mathrm{CSFlow}}(t)
    =
    \frac{
        \widetilde{w}_{\mathrm{CSFlow}}(t)
    }{
        \int_0^1
        \widetilde{w}_{\mathrm{CSFlow}}(\tau)
        \, d\tau
    }.
\end{align}
In practice, we approximate Eq.~\ref{eq:raw_csf_score} with finite differences using \(\Delta r_{\mathrm{signal}}(f,t,\Delta t)\), and approximate the integral in Eq.~\ref{eq:csflow_weight} as a sum over the chosen timestep grid. 

Fig.~\ref{fig:weights_and_steps} (left) shows the resulting weights for an example setting. 
The largest weights occur in the early-to-middle part of the denoising process, indicating that these timesteps recover the frequencies most aligned with human visual sensitivity. Those continuous weights integrate to $1$, so they can directly be used as a noise levels distribution during training.

\begin{figure}[t]
  \centering
  \begin{minipage}[b]{0.48\linewidth}
    \centering
    \includegraphics[width=\linewidth]{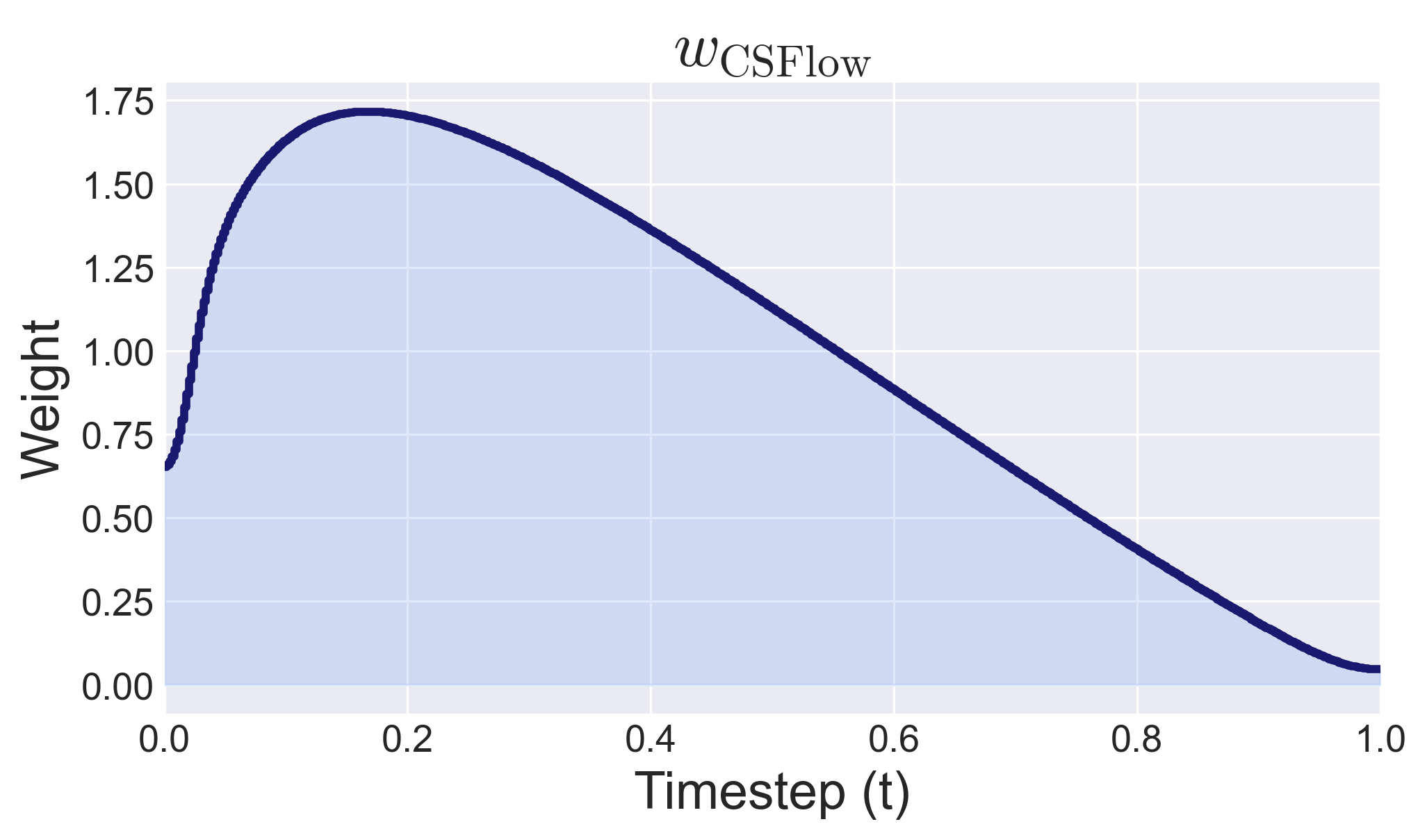}
  \end{minipage}
  \hfill
  \begin{minipage}[b]{0.48\linewidth}
    \centering
    \includegraphics[width=\linewidth]{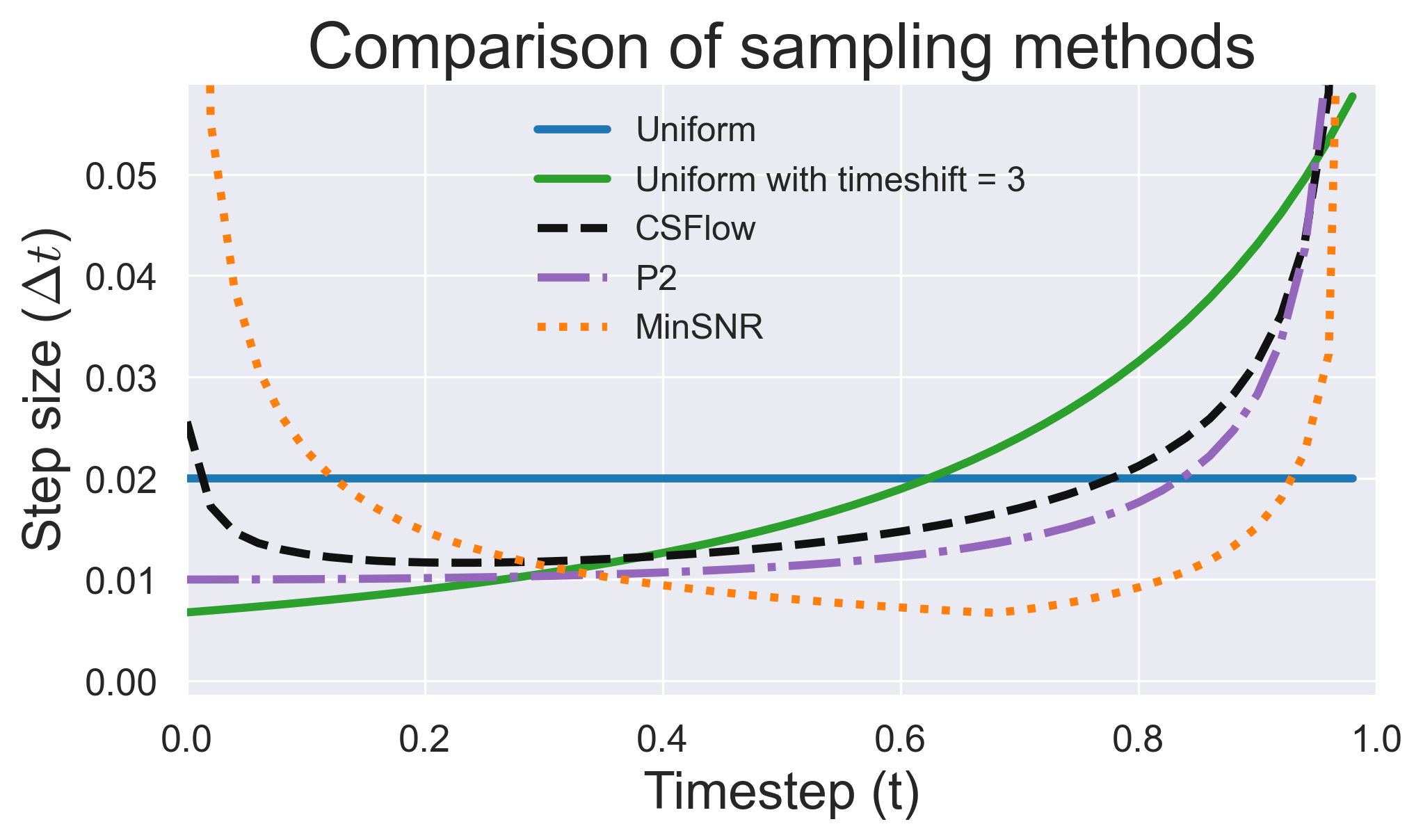}
  \end{minipage}
  \caption{Resulting \(w_{CSFlow}\) (left) and a comparison of sampling methods during inference including our weighted method (right). In the plots, the denoising process happens from left to right. Our pure weights strongly bias the model towards early to mid-range generation stages.
  }
  \label{fig:weights_and_steps}
\end{figure}

\subsection{Getting timesteps from weights}
\label{sec:inverse-cdf}

After the normalization in Eq.~\ref{eq:csflow_weight}, $w_{\mathrm{CSFlow}}$ integrates to one over $[0,1]$ and can be treated as a density over denoising time, with cumulative distribution
\vspace{-0.1cm}
\begin{align}
\label{eq:cdf}
    CDF(t) = \int_0^t w_{\mathrm{CSFlow}}(\tau)\,d\tau .
\end{align}
\vspace{-0.2cm}

We use an inverse of this CDF in two places. For \textbf{training}, we draw $u\sim\mathcal{U}(0,1)$ and set $t = CDF^{-1}(u)$, which samples noise levels proportionally to $w_{\mathrm{CSFlow}}$. For \textbf{inference}, we evaluate the same inverse CDF at even grid of quantiles,

\vspace{-0.2cm}
\begin{align}
\label{eq:inference-grid}
    \tilde{t}_k = CDF^{-1}\!\left(\tfrac{k}{N}\right),
    \qquad k = 0,\dots,N ,
\end{align}
\vspace{-0.2cm}

which gives a grid of $N$ steps satisfying $\int_{\tilde t_k}^{\tilde t_{k+1}} w_{\mathrm{CSFlow}}(t)\,dt = 1/N$ for every $k$. In regions with large $w_{\mathrm{CSFlow}}$ the CDF rises faster, so the inverse gets flatter and the steps are smaller. The grid is more dense in the intervals that recover perceptually important frequencies. Fig.~\ref{fig:weights_and_steps} (right) shows the resulting step sizes against uniform and time-shifted grids. We provide more details on obtaining a smooth monotone inverse CDF, as well as some implementation details in the Appendix~\ref{technical-details}. 

\section{Experiments}
\label{sec:experiments-setup}

To test the weights, we incorporate them in inference and short finetuning of different flow matching models and compare our setup to its original training/inference. We also compare our weighting influence to the commonly adopted min-SNR-$\gamma$~\citep{minsnrgamma} and similarly motivated P2~\citep{p2} weights. For text-to-image generation, we report results on GenEval \citep{geneval}. For class-to-image generation, we report FID \citep{fid} and Inception Score (IS) \citep{is}. We run all single experiments 5 times with different seeds and report the mean score with two standard deviations (95\% confidence interval).

In all experiments, we calculate our weights based on the chosen model's training dataset and forward noise schedule, and then interpolate between the CSFlow weights and the model's base weights according to the interpolation parameter $\alpha$: 

\vspace{-0.2cm}
\begin{align}\label{interpolation}
    w_{final}(t) = \alpha \cdot w_{CSFlow}(t) + (1 - \alpha)\cdot w_{base}(t) \text{,}
\end{align}
\vspace{-0.2cm}

one for training ($\alpha_1$) and one for inference ($\alpha_2$). We provide further details on the alphas used in the experiments in App.~\ref{app:alpha-values} and show their impact during inference in Fig.~\ref{fig:geneval-alpha-sweep}.

\subsection{Model assumptions}

Our derivation assumes that the noising and denoising process is defined directly in pixel space.
This is important because \(r_{\mathrm{signal}}\) relates the noise level at time $t$ to recoverable spatial frequencies in the image.
In latent-space models, noise is applied to learned latent variables, whose x-y coordinates may not correspond explicitly to pixel-space coordinates.
Therefore, the connection between denoising time and recoverable image frequencies is not well defined, and CSFlow cannot be directly applied (we define it as a possible extension to latent models in App.~\ref{app:equalized-gain}).
For this reason, we evaluate our method on pixel-space generative models, namely PixelGen-XXL and PixelGen-XL ~\citep{pixelgen}, PixelREPA-H~\citep{pixelrepa} and JiT-H~\citep{jit}. In all of our experiments, we use the original inference setups (e.g., CFG configuration, number of inference steps) and vary only the timestep weighting in finetuning or inference.

\begin{wraptable}{r}{0.68\textwidth}
\vspace{-0.5cm}
  \caption{Effect of our weighting scheme on text-to-image generation (PixelGen-XXL/16, $512 \times 512$). We report the overall GenEval score. 
  The upper block applies each weighting scheme directly, the lower block interpolates it with the model's base weights and reports the best performing configuration; the corresponding $\alpha_1, \alpha_2$ are listed in Tab.~\ref{tab:alpha-values}. Best result per column and block in \textbf{bold}.}
  \label{tab:main-t2i-results}
  \centering
  \small
  \setlength{\tabcolsep}{4pt}
  \renewcommand{\arraystretch}{1.05}
  \begin{tabular}{lccc}
    \toprule
    Method & Inference-only & Finetuning-only & Finetuning + Inf. \\
    \midrule
    Baseline & 0.791 $\pm$ 0.0036 & 0.789 $\pm$ 0.0048 & 0.789 $\pm$ 0.0048 \\
    \midrule
    \multicolumn{4}{l}{\textit{base weights}} \\
    \quad P2 & 0.778 $\pm$ 0.0048 & 0.809 $\pm$ 0.0062 & \textbf{0.806} $\pm$ 0.0010 \\
    \quad MinSNR & 0.765 $\pm$ 0.0046 & 0.786 $\pm$ 0.0022 & 0.758 $\pm$ 0.0052 \\
    \rowcolor[rgb]{0.90,0.90,0.98}
    \quad CSFlow (ours) & \textbf{0.798} $\pm$ 0.0050 & \textbf{0.810} $\pm$ 0.0008 & 0.797 $\pm$ 0.0022 \\
    \midrule
    \multicolumn{4}{l}{\textit{with interpolation}} \\
    \quad P2 & 0.792 $\pm$ 0.0080 & 0.809 $\pm$ 0.0062 & 0.806 $\pm$ 0.0010 \\
    \quad MinSNR & 0.788 $\pm$ 0.0096 & 0.786 $\pm$ 0.0022 & 0.788 $\pm$ 0.0026 \\
    \rowcolor[rgb]{0.90,0.90,0.98}
    \quad CSFlow (ours) & \textbf{0.798} $\pm$ 0.0050 & \textbf{0.810} $\pm$ 0.0008 & \textbf{0.813} $\pm$ 0.0044 \\
    \bottomrule
  \end{tabular}
  \vspace{-0.2cm}
\end{wraptable}

\subsection{Text-to-Image generation}

\textbf{Setup.} For text-to-image generation we use PixelGen-XXL/16 \citep{pixelgen} at $512 \times 512$ resolution. This model is pretrained with ImageNet for a total of $280K$ steps and further finetuned with high-quality instruction-tuning dataset BLIP3o-60k~\citep{blip3o} for $40K$ steps. We calculate our weights based on BLIP3o-60k dataset statistics. In inference, we use Adams-2nd solver with 25 steps with the modified-weighted timestep grid. 
In finetuning comparisons, we train all models for 7500 steps on the BLIP3o-60k dataset, with weighted noise level distributions and unmodified baselines.

In our main results, we use all weighting schemes as is. Additionally, we find that alpha-interpolated schemes (Eq.  \ref{interpolation}) can improve the benchmark scores. In our additional experiments, we interpolate weights with the original setup (uniform and time-shifted in inference and logit normal in training), obtaining mixed-distributions, and report the best performing $\alpha_1, \alpha_2$ setups.

\textbf{Quantitative results.} Table \ref{tab:main-t2i-results} reports the performance of our CSFlow-modifications in text-to-image generation using GenEval benchmark. The PixelGen-XXL inference-weighted model achieves GenEval score of 0.798 surpassing its unmodified version. Finetuning the model with our weighted time sampling method (or using interpolation \(\alpha_1=0.8\) and \(\alpha_2=0.7\)) further improves the model, achieving GenEval score of 0.810 and 0.813.

\begin{figure}[t]
    \centering
    \includegraphics[width=1.0\linewidth]{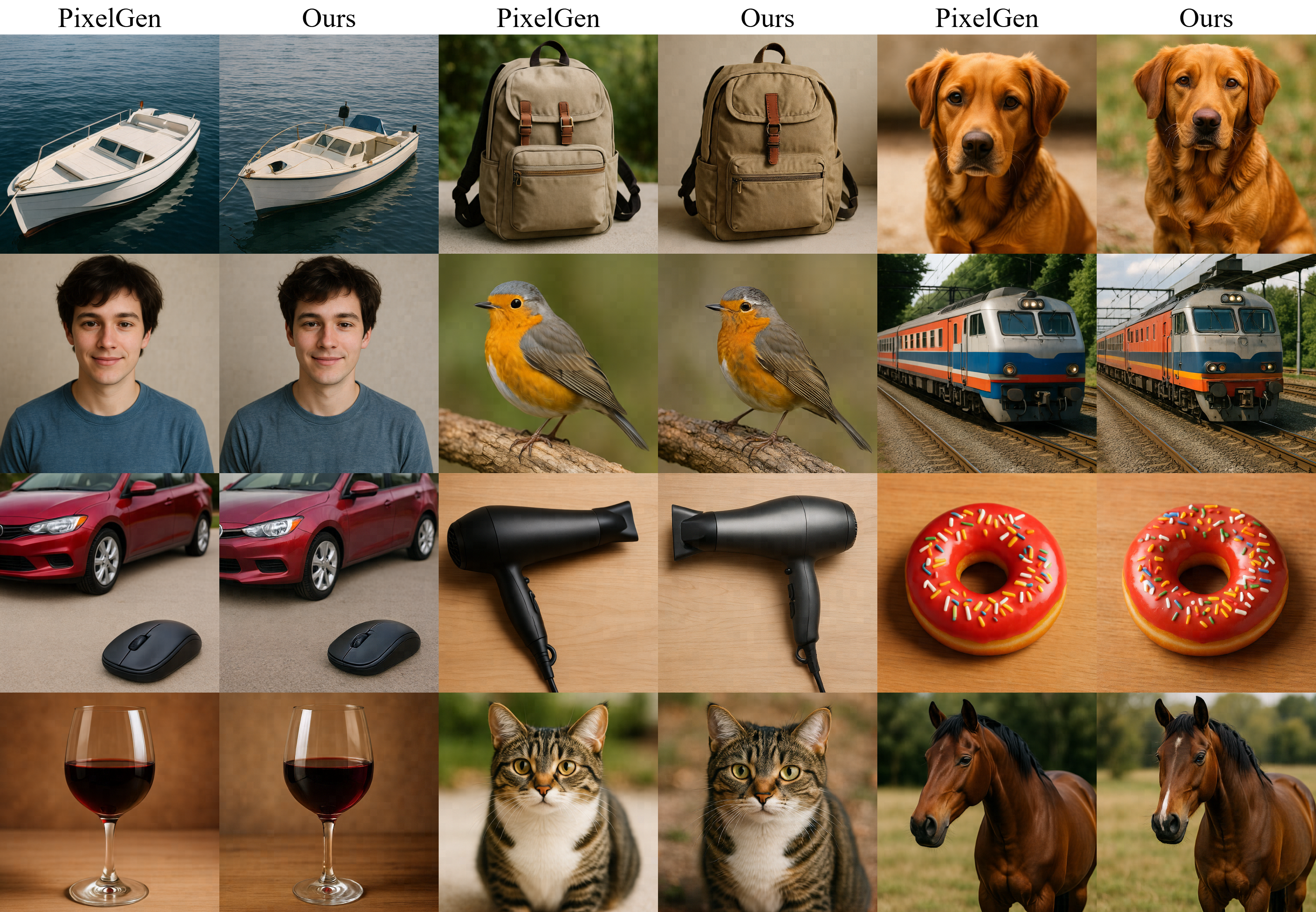}
    \caption{Comparison of the base PixelGen-XXL/16 model (left) with our weighted version (with \(\alpha_1 = 0.8\) and \(\alpha_2 = 0.7\)) (right) in text-to-image generation. All images are 512 \(\times\) 512 resolution.
    We observe that with our weights, samples often yield much less cartoonish-looking generations.
    Best viewed zoomed in. More results can be found in the appendix.}
    \vspace{-0.3cm}
    \label{fig:t2i-qualitative-results}

\end{figure}

\textbf{Qualitative results.} Figure \ref{fig:t2i-qualitative-results} shows chosen image pairs that compare the generations of the baseline model and our best configuration model. We observe that our model generates \textbf{more realistic} images in terms of their textures: while the baseline model suffers from oversmoothing objects and surfaces, we find that our weighting in finetuning and inference fixes that issue, making the images look less artificial and cartoonish. 

\begin{wrapfigure}{rp}{0.58\textwidth}
    \vspace{-1.6cm}
    \centering
    \includegraphics[width=1\linewidth]{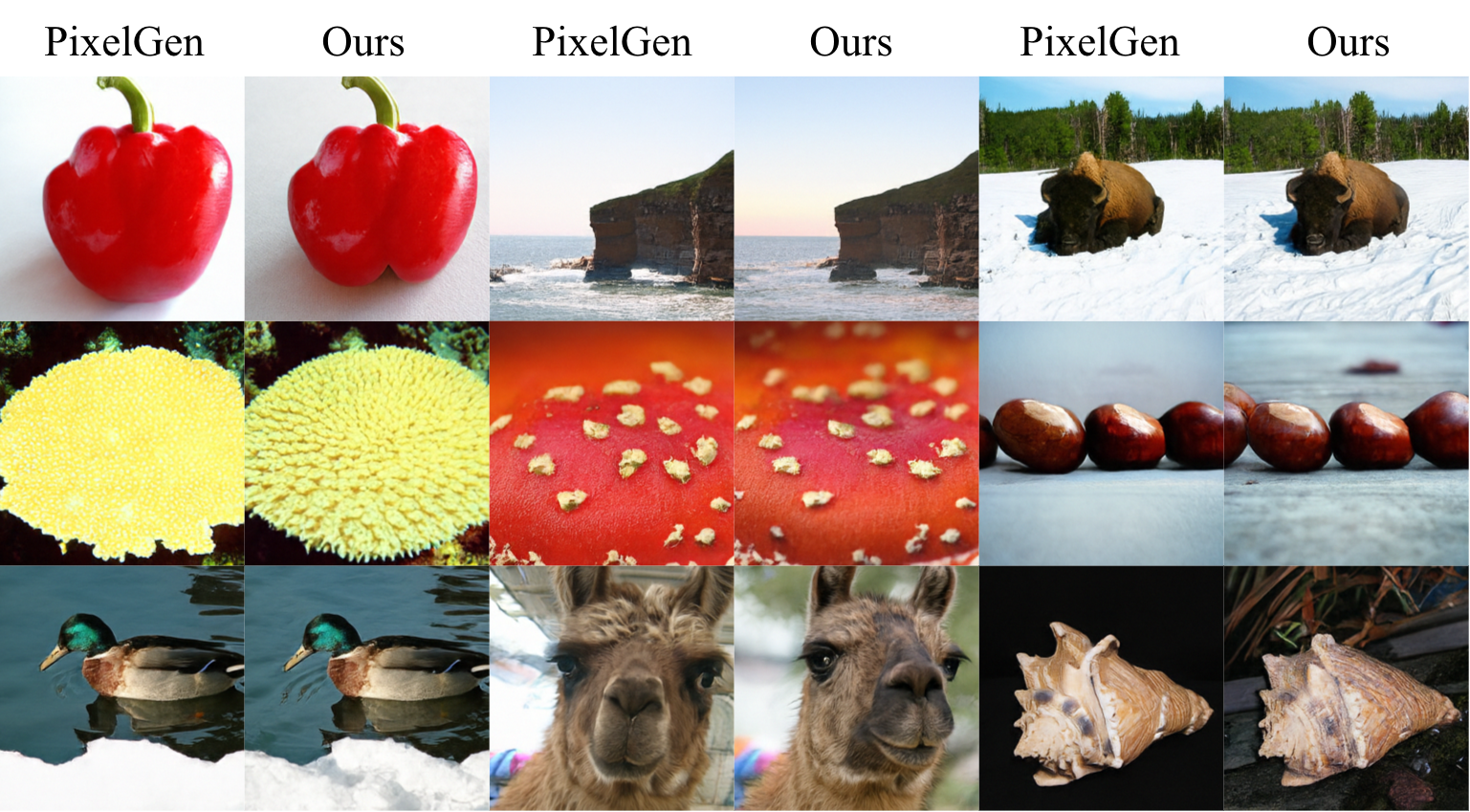}
    \caption{Comparison of the base PixelGen-XL/16 model (left) with our weighted version with \(\alpha_2 = 0.5\)) (right) in class-to-image generation. All images are 256 \(\times\) 256 resolution. While findings are generally the same, the differences are less noticeable in lower quality images.}
    \label{fig:c2i-qualitative-results}
\vspace{-0.2cm}
\end{wrapfigure}

\subsection{Class-to-Image generation}
\textbf{Setup.} For class conditional generation, we use three models: JiT-H/16 \citep{jit}, PixelREPA-H/16 \citep{pixelrepa}, and PixelGen-XL/16 \citep{pixelgen}.
All models are trained on ImageNet \(256 \times 256\) \citep{imagenet1, imagenet2}, JiT and PixelREPA for 600 epochs and PixelGen for 160 epochs. We use the same dataset for analyzing the average data frequency content and use it to calculate our weights. JiT is a relatively straightforward flow model that predicts the clean image $\mathbf{x}_{clean}$ itself and does not use any tokenizer, pre-training or extra losses. 
PixelREPA extends JiT with a modified REPA-style alignment loss, replacing REPA’s per-token MLP projector with a Masked Transformer Adapter. PixelGen instead uses REPA with a regular per-token MLP and output-space perceptual losses on the predicted clean image: LPIPS (VGG)~\citep{lpips} and a cosine loss on last-layer DINOv2-B~\citep{dino} patch features.
In inference, we interpolate our weights with uniform ones and test the resulting non-uniform step sizes. We use Heun sampler with 50 steps for all setups. In training, we sample from a modified-weighted distribution. We finetune the JiT-H model for 2 epochs.

\begin{table}[t]
  \caption{Effect of our weights on class-to-image generation (ImageNet $256 \times 256$ with CFG). We report FID and Inception Score, each block reports one metric. 
  All weighting schemes are interpolated with the base weights at the best performing $\alpha$ per method; values are listed in Tab.~\ref{tab:alpha-values}. Best result per column and block in \textbf{bold}.
  }
  \label{tab:c2i-results}
  \centering
  \small
  \setlength{\tabcolsep}{4.7pt}
  \renewcommand{\arraystretch}{1.05}
  \begin{tabular}{lccccc}
    \toprule
    & \multicolumn{3}{c}{JiT-H} & PixelREPA-H & PixelGen-XL \\
    \cmidrule(lr){2-4} \cmidrule(lr){5-5} \cmidrule(lr){6-6}
    Method & Inference-only & Finetuning-only & Finetuning + Inf. & Inference-only & Inference-only \\
    \midrule
    \multicolumn{6}{l}{
    \textit{FID} $\downarrow$} \\
    \quad Baseline & 1.89 $\pm$ 0.022 & 1.89 $\pm$ 0.038 & 1.89 $\pm$ 0.038 & 1.81 $\pm$ 0.030 & 1.88 $\pm$ 0.034 \\
    \quad P2 & 1.84 $\pm$ 0.022 & \textbf{1.86} $\pm$ 0.042 & \textbf{1.78} $\pm$ 0.030 & 1.62 $\pm$ 0.008 & 1.89 $\pm$ 0.022 \\
    \quad MinSNR & 1.84 $\pm$ 0.022 & 1.90 $\pm$ 0.026 & 1.86 $\pm$ 0.042 & 1.68 $\pm$ 0.018 & 1.88 $\pm$ 0.026 \\
    \rowcolor[rgb]{0.90,0.90,0.98}
    \quad CSFlow (ours) & \textbf{1.80} $\pm$ 0.020 & 1.89 $\pm$ 0.010 & \textbf{1.78} $\pm$ 0.040 & \textbf{1.59} $\pm$ 0.010 & \textbf{1.85} $\pm$ 0.030 \\
    \midrule
    \multicolumn{6}{l}{
    \textit{Inception Score} $\uparrow$} \\
    \quad Baseline & 297.2 $\pm$ 3.20 & 294.1 $\pm$ 3.98 & 294.1 $\pm$ 3.98 & 316.6 $\pm$ 2.30 & 292.7 $\pm$ 6.68 \\
    \quad P2 & 297.4 $\pm$ 3.52 & 298.9 $\pm$ 2.78 & 302.9 $\pm$ 2.06 & 316.7 $\pm$ 4.68 & 295.3 $\pm$ 2.84 \\
    \quad MinSNR & \textbf{302.6} $\pm$ 1.72 & 290.8 $\pm$ 1.88 & 291.6 $\pm$ 1.46 & \textbf{320.9} $\pm$ 1.30 & 294.0 $\pm$ 4.16 \\
    \rowcolor[rgb]{0.90,0.90,0.98}
    \quad CSFlow (ours) & 299.7 $\pm$ 3.58 & \textbf{299.2} $\pm$ 3.56 & \textbf{303.5} $\pm$ 2.76 & 320.8 $\pm$ 4.66  & \textbf{301.5} $\pm$ 2.54 \\
    \bottomrule
  \end{tabular}
  \vspace{-0.5cm}
\end{table}

\textbf{Quantitative results.} Table \ref{tab:c2i-results} reports the performance of our weighted-modifications in class-to-image generation using FID and IS metrics. The PixelREPA-H inference-weighted model achieves the FID score of 1.59 and IS score of 320.8, significantly improving the baseline model (FID 1.81, IS 316.6). The JiT-H and PixelGen-XL inference-weighted models also improve scores of the baseline model (FID 1.89$\rightarrow$1.80 and 1.88$\rightarrow$1.85, IS 297.2$\rightarrow$299.7 and 292.7$\rightarrow$301.5). To emphasize, by \textbf{only introducing our weights in inference}, we surpass the JiT baseline model, as well as its largest version JiT-G/16 with over 1B more parameters (FID 1.82, IS 292.6 reported by JiT authors). By combining finetuning and weighted-inference in JiT, we achieve even better FID of 1.78 and IS 303.5.

\textbf{Qualitative results.} On ImageNet generation, we observe similar effects as in high resolution text-to-image generation: more realistic textures and better color tones and contrast. However, due to the general low quality of generated images, the effect is less significant. In Figure \ref{fig:c2i-qualitative-results} we present pair comparisons between the baseline model and our modified version of it that reflect our observations.

\subsection{Analysis}\label{subsec:exp:analysis}
We further analyze the GenEval scores of different methods over different interpolation states between original and proposed weights in Fig.~\ref{fig:geneval-alpha-sweep}. We can see that, in contrast to other weighting schemes, CSFlow weighting in inference tend to improve scores when approaching $\alpha_2=1.0$ -- directly applied weights. Fig.~\ref{fig:visual_alpha_sweep} shows a visual analysis of two examples when varying $\alpha_1=\alpha_2$. We observe that the realism gradually increases over the trajectory. We note though, that while this pattern of improvement over higher alpha is consistent in our GenEval experiments, in our ImageNet experiments interpolated alphas have shown significantly better performance, with the best reported in Tab.~\ref{tab:alpha-values}.

\textbf{Frequency reconstruction.}
We also verify that our method leads to improvements in modeling the overall frequency distribution.
Comparing the radially averaged power spectral density (RAPSD) of generated images against that of the data, we find that the log-space power spectrum error of CSFlow generations is lower than the baseline's across the whole spectrum, and that the reduction is largest in the mid-frequency band where the CSF is high (a $33\%$ to $51\%$ relative reduction there, against $1\%$ to $6\%$ at the spectral extremes). The full analysis is given in Appendix~\ref{app:spectral}.

\begin{figure}[t]
  \centering
  \begin{minipage}[b]{0.48\linewidth}
    \centering
    \includegraphics[width=\linewidth]{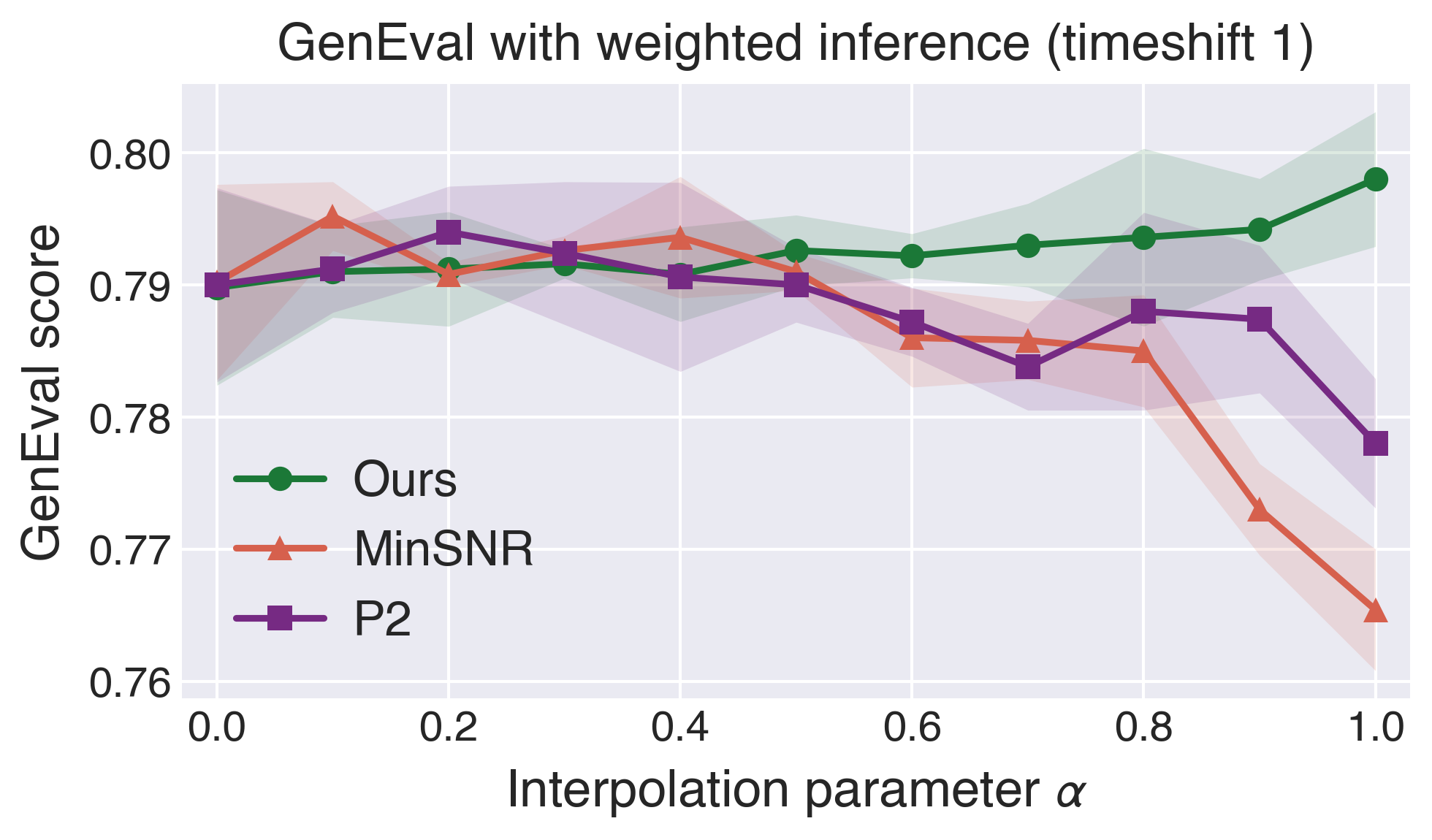}
  \end{minipage}
  \hfill
  \begin{minipage}[b]{0.48\linewidth}
    \centering
    \includegraphics[width=\linewidth]{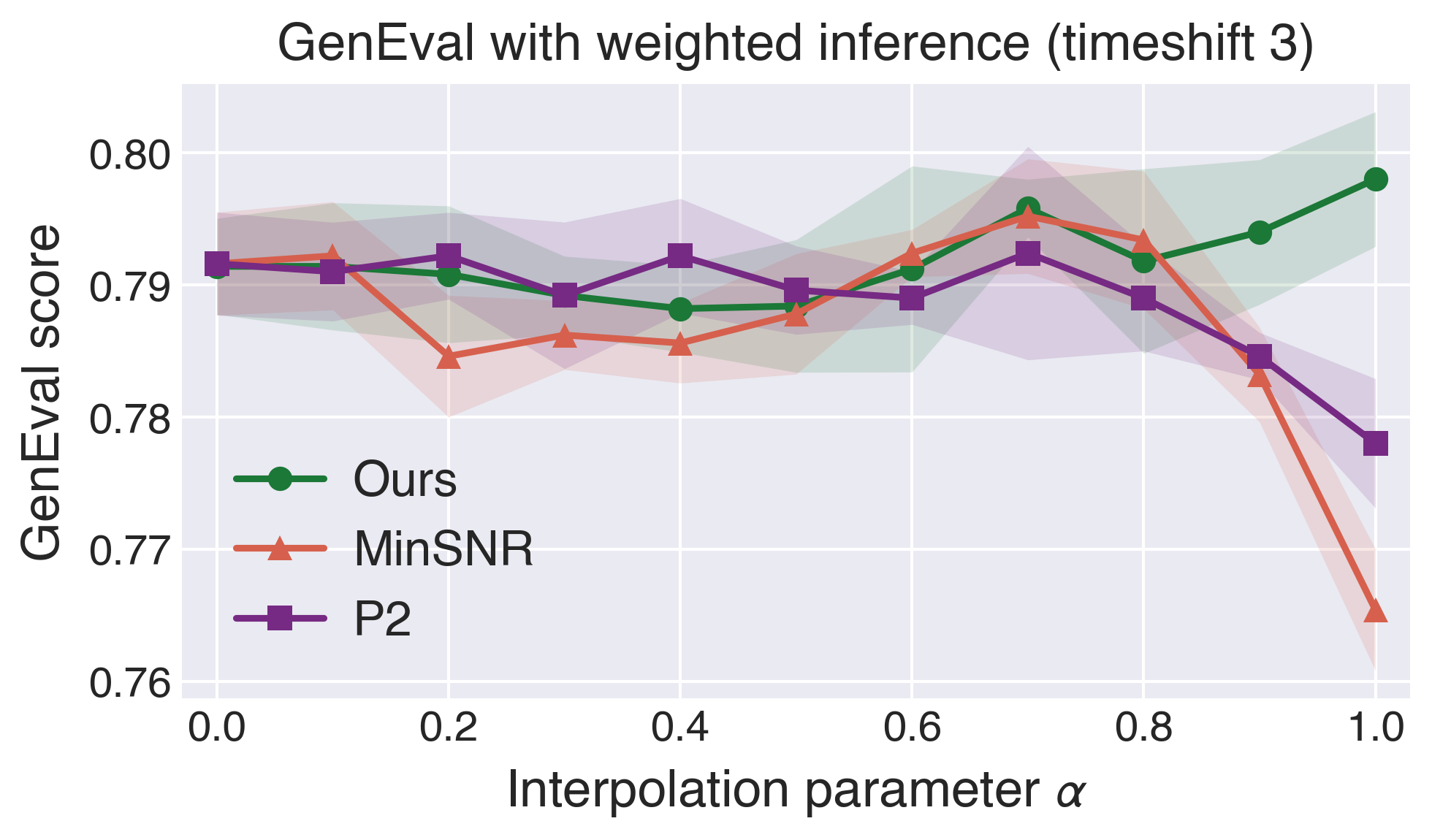}
  \end{minipage}

  \caption{Change in GenEval score across interpolation parameter \(\alpha\) in inference for all three perception-driven methods. We interpolate with uniform (left) and shifted uniform weights (right). Our method surpasses other two in the best setups (when using weights directly with \(\alpha = 1.0\)). Point values show the mean GenEval score and shaded regions show 95\% confidence interval.}
  \label{fig:geneval-alpha-sweep}
\end{figure}

\begin{figure}[t]
    \centering
    \includegraphics[width=1\linewidth]{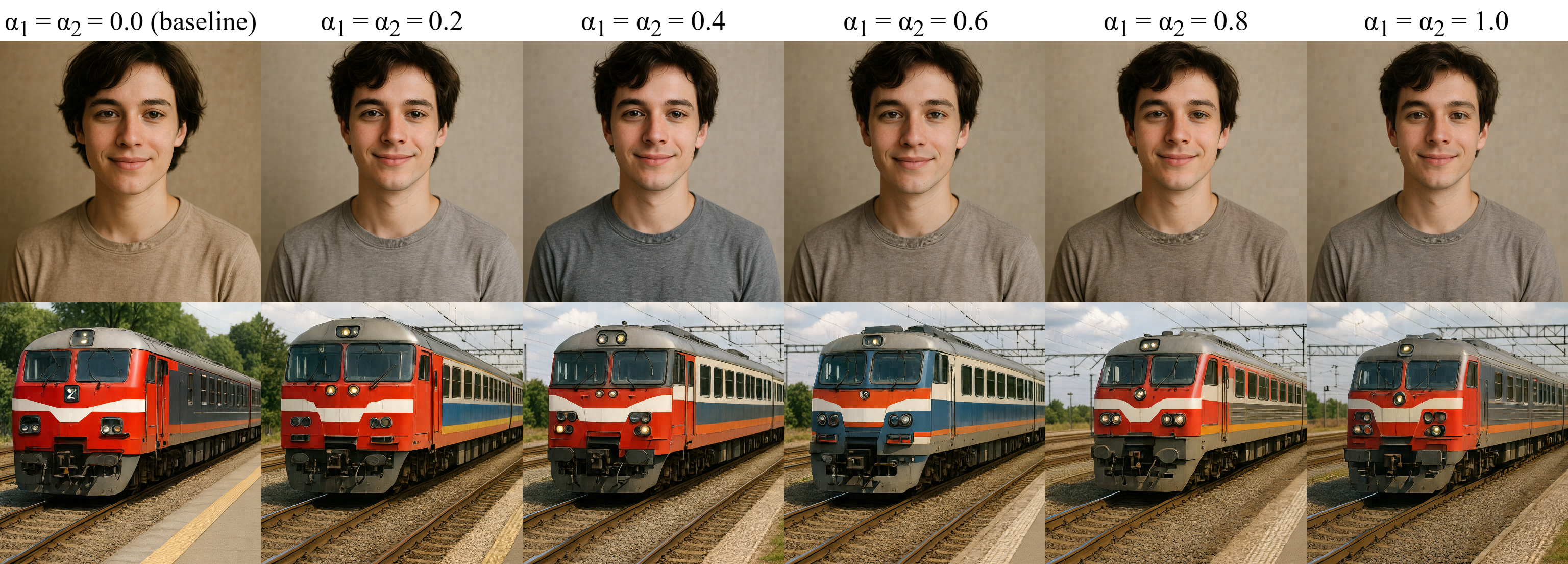}
    \caption{Change in generated images for increasing $\alpha$, interpolating from uniform weights $\alpha_1=\alpha_2=0.0$ to our perception-driven weights $\alpha_1=\alpha_2=1.0$ for two fixed seeds. We observe that images become visually more realistic when increasing $\alpha$. Best viewed zoomed in.}
    \label{fig:visual_alpha_sweep}
    \vspace{-0.4cm}
\end{figure}

\subsection{User preference study}
\label{sec:user-study}

To test whether the improvement in visual realism is visible to human observers, we conducted a two-stage preference study. In the first stage, participants voted on the \textbf{correctness} of individual generated images taken from the full pool of the original and CSFlow-weighted PixelGen-XXL generations. 
We then filtered out image pairs, in which either the baseline or our generation was marked as incorrect (e.g. missing legs, artifacted). 
Filtering rates were balanced across the two models ($54\%$ for ours, $55\%$ for the baseline, no statistically significant difference), so the comparison is not biased by the filter. In the second stage, the retained pairs were shown in randomized order and left/right position, and participants chose the image they preferred in terms of \textbf{quality}, \textbf{realism} and \textbf{aesthetics}. Across 393 pairwise judgments, our generations were preferred in $\mathbf{65\%}$ of comparisons ($95\%$ Wilson confidence interval: $[60.0\%, 69.4\%]$). Furthermore, $69.4\%$ of 49 participants preferred our generations in the majority of the pairs they judged, against $20.4\%$ for the baseline, thus pair-level preference was not driven by a small number of strongly opinionated users. The full protocol, verbatim instructions and more analysis are given in Appendix~\ref{app:user-study}.

\section{Conclusion}
We introduced Contrast Sensitive Flow (CSFlow), a procedure to weigh the information gain of individual steps in flow matching using human contrast sensitivity. We are the first to explicitly describe which frequencies are \textit{being generated} at specific $t$'s. In experiments on text-to-image and class-to-image generation tasks, we were able to verify that the introduction of such weights leads to quantitative improvements in training-free and finetuning settings. Qualitatively, images generated with the derived weights show higher visual realism.

\paragraph{Limitations and Future Work.}
While CSFlow improves the visual realism of generated images, we find that it does not impact the number of large geometric errors (e.g. missing legs) which are sometimes made by the base models. This suggests that the reasons for these types of errors are different from the suboptimal weighting of timesteps. A conceptual limitation of our weighting derivation is the required assumption of a typical screen and screen distance, leading to the necessity of calibration if fully optimal weights are desired. A promising venture for the future would be to analyze how image frequencies translate into common VAE latent spaces to transfer appropriate weighting schemes to latent image generation. 

\section*{Acknowledgments}
Jan Eric Lenssen and Bart Pogodzinski are supported by the German Research Foundation (DFG) - 556415750 (Emmy Noether Programme, project: Spatial Modeling and Reasoning).

\section*{Disclosures}
The authors declare no competing interests.

\bibliography{references}

@article{jit,
  author  = {Li, Tianhong and He, Kaiming},
  title   = {Back to Basics: Let Denoising Generative Models Denoise},
  journal = {arXiv preprint arXiv:2511.13720},
  year    = {2025}
}

@article{pixelgen,
  author  = {Ma, Zehong and Xu, Ruihan and Zhang, Shiliang},
  title   = {{PixelGen}: Pixel Diffusion Beats Latent Diffusion with Perceptual Loss},
  journal = {arXiv preprint arXiv:2602.02493},
  year    = {2026}
}

@article{ediffi,
  author  = {Balaji, Yogesh and Nah, Seungjun and Huang, Xun and Vahdat, Arash and Song, Jiaming and Zhang, Qinsheng and Kreis, Karsten and Aittala, Miika and Aila, Timo and Laine, Samuli and Catanzaro, Bryan and Karras, Tero and Liu, Ming-Yu},
  title   = {{eDiff-I}: Text-to-Image Diffusion Models with an Ensemble of Expert Denoisers},
  journal = {arXiv preprint arXiv:2211.01324},
  year    = {2022}
}

@inproceedings{minsnrgamma,
  author    = {Hang, Tiankai and Gu, Shuyang and Li, Chen and Bao, Jianmin and Chen, Dong and Hu, Han and Geng, Xin and Guo, Baining},
  title     = {Efficient Diffusion Training via {Min-SNR} Weighting Strategy},
  booktitle = {Proceedings of the IEEE/CVF International Conference on Computer Vision (ICCV)},
  pages     = {7441--7451},
  year      = {2023}
}

@inproceedings{imagenet1,
  author    = {Deng, Jia and Dong, Wei and Socher, Richard and Li, Li-Jia and Li, Kai and Fei-Fei, Li},
  title     = {{ImageNet}: A Large-Scale Hierarchical Image Database},
  booktitle = {2009 IEEE Conference on Computer Vision and Pattern Recognition},
  pages     = {248--255},
  year      = {2009},
  doi       = {10.1109/CVPR.2009.5206848}
}

@article{imagenet2,
Author = {Olga Russakovsky and Jia Deng and Hao Su and Jonathan Krause and Sanjeev Satheesh and Sean Ma and Zhiheng Huang and Andrej Karpathy and Aditya Khosla and Michael Bernstein and Alexander C. Berg and Li Fei-Fei},
Title = {{ImageNet Large Scale Visual Recognition Challenge}},
Year = {2015},
journal   = {International Journal of Computer Vision (IJCV)},
doi = {10.1007/s11263-015-0816-y},
volume={115},
number={3},
pages={211-252}
}

@inproceedings{geneval,
  author    = {Ghosh, Dhruba and Hajishirzi, Hannaneh and Schmidt, Ludwig},
  title     = {{GenEval}: An Object-Focused Framework for Evaluating Text-to-Image Alignment},
  booktitle = {Advances in Neural Information Processing Systems},
  volume    = {36},
  note      = {Datasets and Benchmarks Track},
  year      = {2023}
}

@inproceedings{lpips,
  author    = {Zhang, Richard and Isola, Phillip and Efros, Alexei A. and Shechtman, Eli and Wang, Oliver},
  title     = {The Unreasonable Effectiveness of Deep Features as a Perceptual Metric},
  booktitle = {Proceedings of the IEEE Conference on Computer Vision and Pattern Recognition (CVPR)},
  pages     = {586--595},
  year      = {2018}
}

@article{dino,
  author  = {Oquab, Maxime and Darcet, Timoth{\'e}e and Moutakanni, Th{\'e}o and Vo, Huy and Szafraniec, Marc and Khalidov, Vasil and Fernandez, Pierre and Haziza, Daniel and Massa, Francisco and El-Nouby, Alaaeldin and Assran, Mahmoud and Ballas, Nicolas and Galuba, Wojciech and Howes, Russell and Huang, Po-Yao and Li, Shang-Wen and Misra, Ishan and Rabbat, Michael and Sharma, Vasu and Synnaeve, Gabriel and Xu, Hu and J{\'e}gou, Herv{\'e} and Mairal, Julien and Labatut, Patrick and Joulin, Armand and Bojanowski, Piotr},
  title   = {{DINOv2}: Learning Robust Visual Features without Supervision},
  journal = {Transactions on Machine Learning Research},
  year    = {2024}
}

@article{blip3o,
  author  = {Chen, Jiuhai and Xu, Zhiyang and Pan, Xichen and Hu, Yushi and Qin, Can and Goldstein, Tom and Huang, Lifu and Zhou, Tianyi and Xie, Saining and Savarese, Silvio and Xue, Le and Xiong, Caiming and Xu, Ran},
  title   = {{BLIP3-o}: A Family of Fully Open Unified Multimodal Models---Architecture, Training and Dataset},
  journal = {arXiv preprint arXiv:2505.09568},
  year    = {2025}
}

@inproceedings{fid,
  author    = {Heusel, Martin and Ramsauer, Hubert and Unterthiner, Thomas and Nessler, Bernhard and Hochreiter, Sepp},
  title     = {{GANs} Trained by a Two Time-Scale Update Rule Converge to a Local {Nash} Equilibrium},
  booktitle = {Advances in Neural Information Processing Systems},
  volume    = {30},
  pages     = {6626--6637},
  year      = {2017}
}

@inproceedings{is,
  author    = {Salimans, Tim and Goodfellow, Ian and Zaremba, Wojciech and Cheung, Vicki and Radford, Alec and Chen, Xi},
  title     = {Improved Techniques for Training {GANs}},
  booktitle = {Advances in Neural Information Processing Systems},
  volume    = {29},
  pages     = {2226--2234},
  year      = {2016}
}

@inproceedings{sd3,
  author    = {Esser, Patrick and Kulal, Sumith and Blattmann, Andreas and Entezari, Rahim and M{\"u}ller, Jonas and Saini, Harry and Levi, Yam and Lorenz, Dominik and Sauer, Axel and Boesel, Frederic and Podell, Dustin and Dockhorn, Tim and English, Zion and Rombach, Robin},
  title     = {Scaling Rectified Flow Transformers for High-Resolution Image Synthesis},
  booktitle = {Proceedings of the 41st International Conference on Machine Learning},
  series    = {Proceedings of Machine Learning Research},
  volume    = {235},
  pages     = {12606--12633},
  publisher = {PMLR},
  year      = {2024}
}

@article{rapsd,
  author  = {Ruzanski, Evan and Chandrasekar, V.},
  title   = {Scale Filtering for Improved Nowcasting Performance in a High-Resolution {X}-Band Radar Network},
  journal = {IEEE Transactions on Geoscience and Remote Sensing},
  volume  = {49},
  number  = {6},
  pages   = {2296--2307},
  year    = {2011},
  doi     = {10.1109/TGRS.2010.2103946}
}

@inproceedings{csf,
  author    = {Barten, Peter},
  title     = {Formula for the Contrast Sensitivity of the Human Eye},
  booktitle = {Image Quality and System Performance},
  series    = {Proceedings of SPIE},
  volume    = {5294},
  pages     = {231--238},
  year      = {2003},
  doi       = {10.1117/12.537476}
}

@InProceedings{dcae15,
  author    = {Chen, Junyu and Zou, Dongyun and He, Wenkun and Chen, Junsong and Xie, Enze and Han, Song and Cai, Han},
  title     = {{DC-AE 1.5}: Accelerating Diffusion Model Convergence with Structured Latent Space},
  booktitle = {Proceedings of the IEEE/CVF International Conference on Computer Vision (ICCV)},
  month     = {October},
  year      = {2025},
  pages     = {19628--19637}
}

@misc{spectral,
  author = {Dieleman, Sander},
  title  = {Diffusion Is Spectral Autoregression},
  year   = {2024},
  note   = {Blog post},
  url    = {https://sander.ai/2024/09/02/spectral-autoregression.html}
}

@article{fourier,
  author  = {Falck, Fabian and Pandeva, Teodora and Zahirnia, Kiarash and Lawrence, Rachel and Turner, Richard and Meeds, Edward and Zazo, Javier and Karmalkar, Sushrut},
  title   = {A Fourier Space Perspective on Diffusion Models},
  journal = {arXiv preprint arXiv:2505.11278},
  year    = {2025}
}

@inproceedings{sid,
  author    = {Hoogeboom, Emiel and Heek, Jonathan and Salimans, Tim},
  title     = {Simple Diffusion: End-to-End Diffusion for High Resolution Images},
  booktitle = {Proceedings of the 40th International Conference on Machine Learning},
  series    = {Proceedings of Machine Learning Research},
  volume    = {202},
  pages     = {13213--13232},
  publisher = {PMLR},
  year      = {2023}
}

@inproceedings{diffusion1,
  author    = {Sohl-Dickstein, Jascha and Weiss, Eric A. and Maheswaranathan, Niru and Ganguli, Surya},
  title     = {Deep Unsupervised Learning Using Nonequilibrium Thermodynamics},
  booktitle = {Proceedings of the 32nd International Conference on Machine Learning},
  series    = {Proceedings of Machine Learning Research},
  volume    = {37},
  pages     = {2256--2265},
  publisher = {PMLR},
  year      = {2015}
}

@inproceedings{ddpm,
  author    = {Ho, Jonathan and Jain, Ajay N. and Abbeel, Pieter},
  title     = {Denoising Diffusion Probabilistic Models},
  booktitle = {Advances in Neural Information Processing Systems},
  volume    = {33},
  pages     = {6840--6851},
  year      = {2020}
}

@inproceedings{iddpm,
  author    = {Nichol, Alex and Dhariwal, Prafulla},
  title     = {Improved Denoising Diffusion Probabilistic Models},
  booktitle = {Proceedings of the 38th International Conference on Machine Learning},
  series    = {Proceedings of Machine Learning Research},
  volume    = {139},
  pages     = {8162--8171},
  publisher = {PMLR},
  year      = {2021}
}

@inproceedings{lipman2022flow,
  title = {Flow Matching for Generative Modeling},
  author = {Yaron Lipman and Ricky T. Q. Chen and Heli Ben-Hamu and Maximilian Nickel and Matthew Le},
  booktitle = {International Conference on Learning Representations {(ICLR)}},
  year = {2023}
}

@inproceedings{liu2022flow,
  title = {Flow straight and fast: Learning to generate and transfer data with rectified flow},
  author = {Liu, Xingchao and Gong, Chengyue and Liu, Qiang},
  booktitle = {International Conference on Learning Representations {(ICLR)}},
  year = {2023}
}

@inproceedings{edm,
  author    = {Karras, Tero and Aittala, Miika and Aila, Timo and Laine, Samuli},
  title     = {Elucidating the Design Space of Diffusion-Based Generative Models},
  booktitle = {Advances in Neural Information Processing Systems},
  volume    = {35},
  pages     = {26565--26577},
  year      = {2022}
}

@inproceedings{ye2025schedule,
  author    = {Ye, Zilyu and Chen, Zhiyang and Li, Tiancheng and Huang, Zemin and Luo, Weijian and Qi, Guo-Jun},
  title     = {Schedule On the Fly: Diffusion Time Prediction for Faster and Better Image Generation},
  booktitle = {Proceedings of the IEEE/CVF Conference on Computer Vision and Pattern Recognition (CVPR)},
  year      = {2025}
}

@inproceedings{p2,
  author    = {Choi, Jooyoung and Lee, Jungbeom and Shin, Chaehun and Kim, Sungwon and Kim, Hyunwoo and Yoon, Sungroh},
  title     = {Perception Prioritized Training of Diffusion Models},
  booktitle = {Proceedings of the IEEE/CVF Conference on Computer Vision and Pattern Recognition (CVPR)},
  pages     = {11472--11481},
  year      = {2022}
}

@misc{pixelrepa,
      title={Representation Alignment for Just Image Transformers is not Easier than You Think}, 
      author={Jaeyo Shin and Jiwook Kim and Hyunjung Shim},
      year={2026},
      eprint={2603.14366},
      archivePrefix={arXiv},
      primaryClass={cs.CV},
      url={https://arxiv.org/abs/2603.14366}, 
}


\appendix
\newpage

\section{Broader Impact}\label{app:broader-impact}

Our work presents a method that improves the realism of generated images. Methods for realistic image generation have the potential to be misused; for example, for the creation of deepfake images. Our contribution is of conceptual nature with proof of concept experiments. We do not provide usable software that goes beyond the capabilities of state-of-the-art image generators from industry.

\section{Contrast Sensitivity Function (CSF)}\label{csf-details}

\subsection{CSF definition}

CSF is defined for spatial frequencies measured in cycles per degree of visual angle. We use Barten's~\cite{csf} formulation and parameters given below.
\begin{equation} \label{csf_eq}
    \text{CSF}(f) =
    \frac{M_{\mathrm{opt}}(f)}{k}
    \left[
    \frac{2}{T}
    \left(
    \frac{1}{X_{0}^{2}} + \frac{1}{X_{\max}^{2}} + \frac{f^{2}}{N_{\max}^{2}}
    \right)
    \left(
    \frac{1}{n p E} +
    \frac{\phi_{0}}{1 - \exp\!\left(-\left(\frac{f}{u_{0}}\right)^{2}\right)}
    \right)
    \right]^{-1/2}
\end{equation}
\begin{equation}
    M_{\mathrm{opt}}(f) = \exp\!\left(-2\pi^{2}\sigma^{2}f^{2}\right)
\end{equation}

\begin{align}
\sigma &= \frac{0.5}{60} , &
k &= 3.0 , &
T &= 0.1 , \\
X_0 &= 60 , &
X_{\max} &= 12 , &
N_{\max} &= 15 , \\
n &= 0.03 , &
p &= 1.2\times 10^{6} , &
E &= 500 , \\
\phi_0 &= 3\times 10^{-8} , &
u_0 &= 7
\end{align}

In the digital images spatial frequencies are defined in the cycles per pixel domain and depend on the physical size of a pixel and viewing distance to the screen. In all usages of CSF we assume the user is looking at the generated images on a regular size laptop in a viewing distance of $50$ cm and a pixel size of $0.0114$ cm which translates to approximately $3$cm and $6$ cm image edge length in $256\times256$ and $512\times512$ resolutions respectively.

\begin{figure}[h]
  \centering
  \begin{minipage}[b]{0.48\linewidth}
    \centering
    \includegraphics[width=\linewidth]{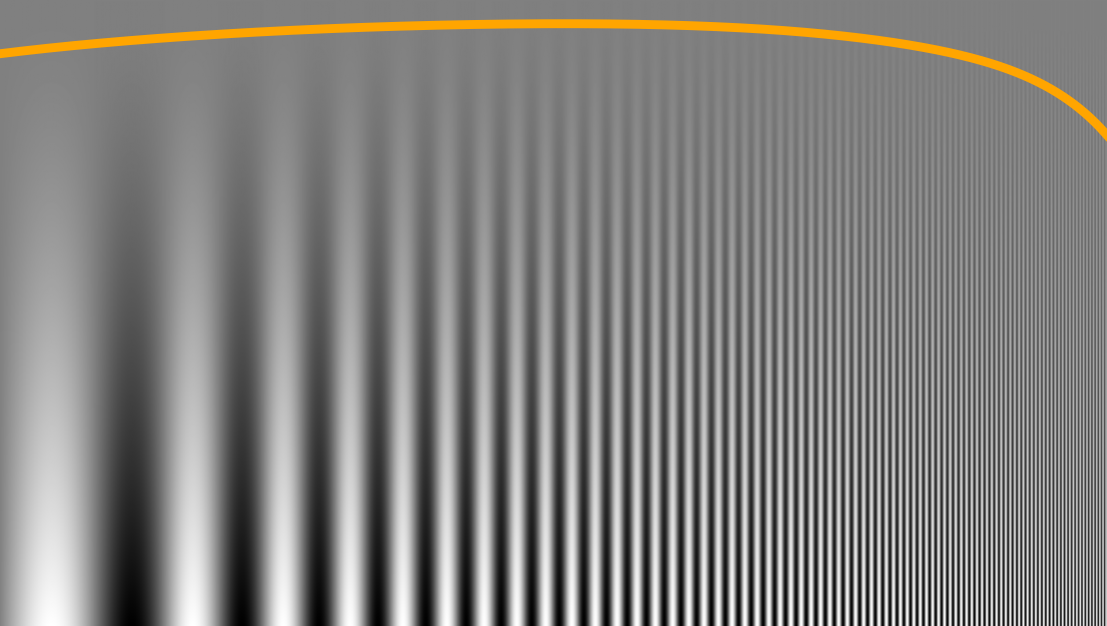}
  \end{minipage}
  \hfill
  \begin{minipage}[b]{0.48\linewidth}
    \centering
    \includegraphics[width=\linewidth]{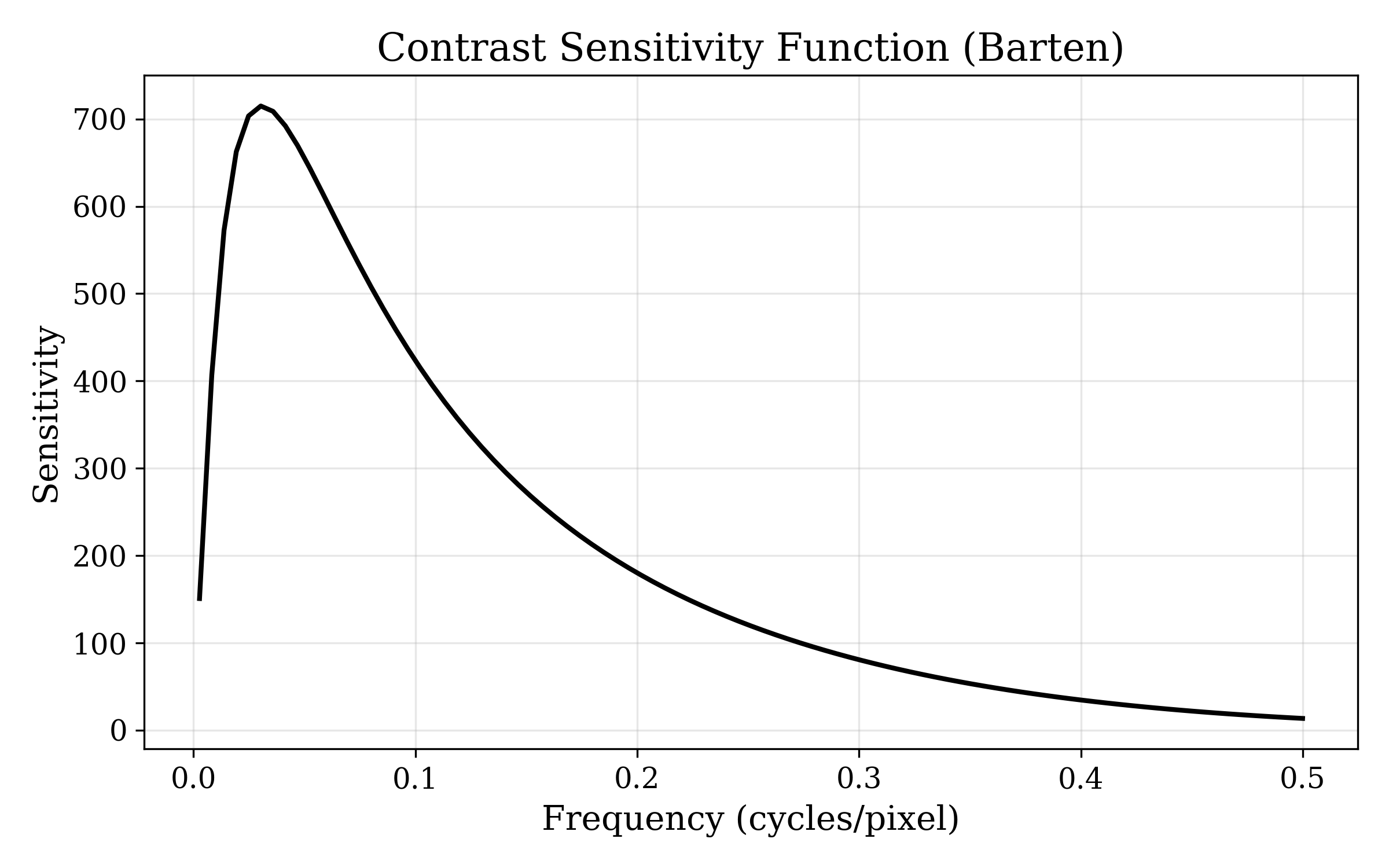}
  \end{minipage}
  \caption{Visualization of CSF (left) and CSF values for spatial frequencies in cycles/pixel domain (right). We plot CSF in log-log space on top of spatial frequencies to visualize the idea of CSF. We assume viewing distance of $40\ cm$ and physical size of the plot $7 \ cm$ (\(\approx \frac{1}{3}\) of an A4 page). }
  \label{fig:csf}
\end{figure}

\subsection{Conversion from cycles per pixel to cycles per degree}
\label{app:csf-conversion}

The CSF is defined over spatial frequencies in cycles per degree of visual angle, while our Fourier frequencies are measured in cycles per pixel. 
Let \(p\) be the physical pixel size and \(d\) the viewing distance. 
The visual angle subtended by one pixel is
\[
    \theta_{\mathrm{pix}}
    =
    2\arctan\left(\frac{p}{2d}\right)
    \cdot \frac{180}{\pi}
    \quad \text{degrees/pixel}.
\]
Therefore, an image frequency \(f\) measured in cycles per pixel corresponds to
\[
    f_{\mathrm{cpd}}
    =
    \frac{f}{\theta_{\mathrm{pix}}}
\]
cycles per degree. 
In our experiments, we use \(p=0.0114\) cm and \(d=50\) cm.

\section{Technical details}
\label{technical-details}

\subsection{Estimating the data power spectrum}
\label{rapsd-details}

The retained signal \(r_{\mathrm{signal}}(f,t)\) requires the expected image power spectrum
\begin{align}
    S_f = \mathbb{E}\left[|F_f(\mathbf{x}_{clean})|^2\right].
\end{align}
    
We estimate \(S_f\) from the model's training dataset using the Radially Averaged Power Spectral Density (RAPSD)~\cite{rapsd}. 
We use RAPSD because our weighting depends only on the spatial frequency magnitude, not on the orientation of the corresponding Fourier component.

For an image \(\mathbf{x}_{clean}\), let
\begin{align}
    P_{\mathbf{x}_{clean}}(u,v)
    =
    |F_{\mathbf{x}_{clean}}(u,v)|^2
\end{align}
    
denote the two-dimensional DFT power at frequency coordinate \((u,v)\). 
For a radial frequency bin \(f\), let \(\mathcal{A}_f\) be the set of Fourier coordinates whose distance from the origin falls into that bin. 
The RAPSD is
\begin{align}
    \mathrm{RAPSD}_{\mathbf{x}_{clean}}(f)
    =
    \frac{1}{|\mathcal{A}_f|}
    \sum_{(u,v)\in \mathcal{A}_f}
    P_{\mathbf{x}_{clean}}(u,v).
\end{align}
We then estimate the expected image power as
\begin{align}
    S_f
    =
    \frac{1}{|D|}
    \sum_{\mathbf{x}_{clean}\in D}
    \mathrm{RAPSD}_{\mathbf{x}_{clean}}(f).
\end{align}
Using this estimate, the retained signal is computed as
\begin{align}
    r_{\mathrm{signal}}(f,t)
    =
    \frac{
        a_t^2 S_f
    }{
        a_t^2 S_f + b_t^2
    }.
\end{align}

\subsection{Computing CSF-based weights}
\label{app:computing-csflow-weights}

In the main text, we define the CSF-based weights in continuous time using the instantaneous increase
\[
    \partial_t r_{\mathrm{signal}}(f,t).
\]
The raw perceptual score at time \(t\) is
\begin{equation}
\label{eq:app_raw_csf_score}
    \widetilde{w}_{\mathrm{CSFlow}}(t)
    =
    \frac{
        \int_{\Omega}
        \partial_t r_{\mathrm{signal}}(f,t) \cdot CSF(f) \, df
    }{
        \int_{\Omega}
        \partial_t r_{\mathrm{signal}}(f,t) \, df
    },
\end{equation}
where \(\Omega\) denotes the set of radial frequency bins. 
This quantity is the average CSF value over the frequencies that become newly recoverable around time \(t\).

We normalize the raw scores to have mean one over the denoising trajectory:
\begin{equation}
\label{eq:app_mean_one_weight}
    w_{\mathrm{CSFlow}}(t)
    =
    \frac{
        \widetilde{w}_{\mathrm{CSFlow}}(t)
    }{
        \int_0^1
        \widetilde{w}_{\mathrm{CSFlow}}(\tau)
        \, d\tau
    }.
\end{equation}
Since the denoising interval has length one, this gives
\[
    \int_0^1 w_{\mathrm{CSFlow}}(t)\,dt = 1,
\]
so the average weight over time is unchanged. In practice, we compute interval-level weights on a timestep grid
\[
    0=t_0 < t_1 < \dots < t_T=1,
    \qquad
    \Delta t_i = t_{i+1}-t_i.
\]
For each interval, we approximate the integral of \(\partial_t r_{\mathrm{signal}}\) using the finite difference
\begin{equation}
    \Delta r_{\mathrm{signal}}(f,t_i,\Delta t_i)
    =
    r_{\mathrm{signal}}(f,t_{i+1})
    -
    r_{\mathrm{signal}}(f,t_i).
\end{equation}
This is consistent with the continuous formulation because
\[
    \Delta r_{\mathrm{signal}}(f,t_i,\Delta t_i)
    =
    \int_{t_i}^{t_{i+1}}
    \partial_t r_{\mathrm{signal}}(f,t)\,dt.
\]
The discrete raw score is then
\begin{equation}
\label{eq:app_discrete_raw_csf_score}
    \widetilde{w}_{\mathrm{CSFlow},i}
    =
    \frac{
        \sum_{f\in\Omega}
        \Delta r_{\mathrm{signal}}(f,t_i,\Delta t_i)
        \cdot CSF(f)
    }{
        \sum_{f\in\Omega}
        \Delta r_{\mathrm{signal}}(f,t_i,\Delta t_i)
    }.
\end{equation}
Finally, we normalize the interval weights to have mean one over time:
\begin{equation}
\label{eq:app_discrete_mean_one_weight}
    w_{\mathrm{CSFlow},i}
    =
    \frac{
        \widetilde{w}_{\mathrm{CSFlow},i}
    }{
        \sum_{j=0}^{T-1}
        \Delta t_j
        \widetilde{w}_{\mathrm{CSFlow},j}
    }.
\end{equation}
For uniformly spaced intervals, this reduces to division by the empirical mean:
\[
    w_{\mathrm{CSFlow},i}
    =
    \frac{
        \widetilde{w}_{\mathrm{CSFlow},i}
    }{
        \frac{1}{T}
        \sum_{j=0}^{T-1}
        \widetilde{w}_{\mathrm{CSFlow},j}
    }.
\]

\subsection{Constructing the smooth inverse CDF}
\label{app:smooth-inverse-cdf}

The CSF-based weights define a time-dependent weighting function \(w_{\mathrm{CSFlow}}(t)\). 
After mean-one normalization, we have
\[
    \int_0^1 w_{\mathrm{CSFlow}}(t)\,dt = 1,
\]
so \(w_{\mathrm{CSFlow}}(t)\) can be interpreted as a probability density over denoising time. 
We define its cumulative distribution function as
\[
    CDF(t)
    =
    \int_0^t w_{\mathrm{CSFlow}}(\tau)\,d\tau.
\]
In practice, \(w_{\mathrm{CSFlow}}\) is computed on a finite timestep grid 
\[
    0=t_0<t_1<\dots<t_T=1.
\]
We approximate the CDF at the grid points using numerical integration:
\[
    CDF_i
    =
    \frac{
        \sum_{j=0}^{i-1}
        w_{\mathrm{CSFlow}}(t_j)\Delta t_j
    }{
        \sum_{j=0}^{T-1}
        w_{\mathrm{CSFlow}}(t_j)\Delta t_j
    },
    \qquad
    \Delta t_j=t_{j+1}-t_j.
\]
The denominator ensures that \(CDF_T=1\) numerically. 
We then construct a smooth monotone approximation of the inverse CDF from the pairs
\[
    (CDF_i,t_i)_{i=0}^{T}.
\]
We denote this approximation by \(CDF^{-1}_{\mathrm{smooth}}\).

\subsection{Training modification: weighted time sampling}
\label{app:weighted-training-sampling}

For training, we sample timesteps from the CSF-weighted time distribution using inverse transform sampling. 
We first sample
\[
    u \sim \mathcal{U}(0,1),
\]
and then set
\[
    t = CDF^{-1}_{\mathrm{smooth}}(u).
\]
This biases training toward denoising times with larger \(w_{\mathrm{CSFlow}}(t)\), i.e. toward intervals where the newly recovered frequencies are more perceptually important.

\subsection{Inference modification: weighted step sizes}
\label{app:weighted-inference-steps}

For inference, we use the same smooth inverse CDF to construct a non-uniform sampling grid. 
Given \(N\) inference steps, we evaluate the inverse CDF at evenly spaced quantiles:
\[
    q_k = \frac{k}{N},
    \qquad
    k=0,1,\dots,N,
\]
and define
\[
    \tilde{t}_k
    =
    CDF^{-1}_{\mathrm{smooth}}(q_k).
\]
The resulting grid \(\{\tilde{t}_k\}_{k=0}^{N}\) has step sizes
\[
    \Delta \tilde{t}_k
    =
    \tilde{t}_{k+1}-\tilde{t}_k.
\]
Since larger values of \(w_{\mathrm{CSFlow}}(t)\) make the CDF increase faster, the inverse CDF changes more slowly in those regions. 
Thus, evenly spaced quantiles produce smaller timestep gaps where \(w_{\mathrm{CSFlow}}\) is large, resulting in a denser inference grid in perceptually important denoising intervals.

\subsection{Interpolation parameters}
\label{app:alpha-values}
In some cases (both for our method and the baselines), we found that using the derived weights provided a too-strong bias in the training distribution, or the inference step-sizes, causing performance degradation. To fairly compare the methods, we interpolate the weights with the model's original training timestep distribution or step-sizes from original paper (e.g. uniform or timeshifted), and report the best performing setup per method. In Table~\ref{tab:alpha-values} we report the $\alpha_1$ (training), $\alpha_2$ (inference) values that gave us the best performance for each of the methods, which we used to obtain the main results in Tables~\ref{tab:main-t2i-results} and \ref{tab:c2i-results}.

\begin{table}[h]
  \caption{Interpolation parameters for the results reported in \cref{tab:main-t2i-results} and \cref{tab:c2i-results}. A value of $1.0$ means the weighting scheme was applied directly, without interpolation.}
  \label{tab:alpha-values}
  \centering
  \small
  \setlength{\tabcolsep}{6pt}
  \renewcommand{\arraystretch}{1.05}
  \begin{tabular}{llccc}
    \toprule
    Model & Method & Inference-only & Training-only & Training + Inference \\
    & & $\alpha_2$ & $\alpha_1$ & $(\alpha_1, \alpha_2)$ \\
    \midrule
    \multirow{3}{*}{PixelGen-XXL/16}
      & P2     & 0.7 & 1.0 & (1.0, 1.0) \\
      & MinSNR & 0.4 & 1.0 & (1.0, 0.7) \\
      & CSFlow (ours) & 1.0 & 1.0 & (0.8, 0.7) \\
    \midrule
    \multirow{3}{*}{JiT-H/16}
      & P2     & 0.3 & 1.0 & (1.0, 0.4) \\
      & MinSNR & 0.3 & 1.0 & (1.0, 0.2) \\
      & CSFlow (ours) & 0.4 & 1.0 & (0.4, 0.4) \\
    \midrule
    \multirow{3}{*}{PixelREPA-H/16}
      & P2     & 0.5 & -- & -- \\
      & MinSNR & 0.4 & -- & -- \\
      & CSFlow (ours) & 0.5 & -- & -- \\
    \midrule
    \multirow{3}{*}{PixelGen-XL/16}
      & P2     & 0.3 & -- & -- \\
      & MinSNR & 0.2 & -- & -- \\
      & CSFlow (ours) & 0.5 & -- & -- \\
    \bottomrule
  \end{tabular}
\end{table}

We believe that although our weights match the visual importance, they are not guaranteed to be optimal in terms of generating the best images. 
While some frequencies might be irrelevant for the human eye (and thus are heavily down-weighted by our method), they might still be relevant for the flow model itself. For example, if very low frequencies are not modeled correctly (e.g. images are often too bright, too blue, etc.) then, in intermediate denoising states, the noisy image might become out of distribution for the model. By interpolating with base weights, we search for weights that put more emphasis on more important frequencies without losing too much capacity for other ones.

There is also a second, inference-specific reason when the weights are used only to place the step grid. There, the benefit of taking smaller steps in a perceptually important region is bounded by the quality of the flow the model has already learned: beyond some step density, smaller steps bring no benefits. On the other hand, reducing the density further at the less important regions will likely further harm the generation quality, since inference trajectories are not straight, and coarse ODE steps may cause a large integration error. 
The weights interpolation lets us balance these effects.

\section{User preference study}
\label{app:user-study}

This appendix gives the full protocol and analysis for the study summarized in \cref{sec:user-study}.

\subsection{Protocol}

Both stages were implemented as web applications and distributed by link to volunteer participants at our institution who are not connected to this work. The study was fully anonymous: preferences cannot be linked to any personal identifier. In both stages, participants were shown a brief description of the study and voluntarily proceeded to the task.

The two models used in this study were PixelGen-XXL/16 in its original version and with CSFlow weighted finetuning and inference ($\alpha_1=0.8, \alpha_2 = 0.7$, our best performing model on the GenEval). Participants were instructed to use a laptop or a larger display. Viewing conditions were otherwise not controlled, and all images were shown at their original resolution of $512 \times 512$.

\paragraph{Stage 1: quality filtering.}
The first stage is a single-image classification task whose purpose is to remove failed generations: images with large geometric errors, distorted elements, etc. For every session, 15 random pair indices are sampled from the full pool of generations across both methods (original and with CSFlow) and shown one at a time in random order. Since both methods' generations used the same prompt and seed, these image pairs contain approximately the same content. Within a session, pairs were sampled without replacement, so the pairs may repeat across different participants’ sessions. The instruction was:

\begin{quote}
\emph{Your task is to decide whether the image is \textbf{correct} – contains \textbf{no} artifacts (e.g. \textbf{no} additional animal legs, \textbf{no} missing legs, the \textbf{right number} of heads or fingers), \textbf{no} distorted elements (e.g. merged two objects), has the \textbf{right geometry}, has the \textbf{correct text}, etc. Odd colors or textures are acceptable.}\\
\emph{For each image, choose:}\\
\emph{\textbf{Correct}: image looks acceptable, may contain unnatural texture, shifted color, unreadable small text, etc.}\\
\emph{\textbf{Incorrect}: image has visual errors, e.g. missing limbs, distorted objects, too many heads/fingers, etc.}
\end{quote}

An image pair was passed to Stage 2 only if both images were independently marked \emph{Correct}. Of 381 candidate pairs, 126 passed this filter.

\paragraph{Stage 2: pairwise preference.}
The second stage is a two-alternative forced-choice task on the pairs retained from Stage 1. For each trial, both the pair and its left/right assignment were randomized. Within a session, pairs were sampled without replacement; they may repeat across different participants' sessions. The instruction was:

\begin{quote}
\emph{Your task is to choose which of the two images (Image A and Image B) you prefer in terms of its quality, realism and aesthetics.}
\end{quote}

In this stage, 49 participants took part and 393 pairwise judgments were collected. No participants or judgments were removed as outliers at either stage.

\subsection{Filtering balance}

Because Stage 2 operates only on pairs that survived Stage 1, a filter that passed one model's images more often than the other's could bias the comparison in Stage 2. We analysed the results, and it was not the case: the per-image pass rate was $54\%$ for CSFlow generations and $55\%$ for the baseline. A two-proportion $z$-test finds no significant difference ($z = 0.39$, $p = 0.699$; $95\%$ confidence interval on the difference $[-5.7, +8.5]$ percentage points).

\subsection{Results}
\label{app:user-study-results}

\begin{table}[h]
  \caption{User study results. Pair-level preference is the fraction of all judgements favouring CSFlow; participant-level figures count each participant once, according to whether they favoured a model in the majority of the pairs they judged.}
  \label{tab:user-study}
  \centering
  \small
  \setlength{\tabcolsep}{7pt}
  \renewcommand{\arraystretch}{1.05}
  \begin{tabular}{lc}
    \toprule
    Measure & Value \\
    \midrule
    \multicolumn{2}{l}{\textit{Pair-level}} \\
    \quad Preference for CSFlow & $65.0\%$ \\
    \quad $95\%$ Wilson score interval & $[60.0\%,\ 69.4\%]$ \\
    \midrule
    \multicolumn{2}{l}{\textit{Participant-level}} \\
    \quad Preferred CSFlow & $69.4\%$ \\
    \quad Preferred baseline & $20.4\%$ \\
    \quad Tied & $10.2\%$ \\
    \midrule
    \multicolumn{2}{l}{\textit{Per-participant preference rate}} \\
    \quad Mean / median & $0.655$ / $0.666$ \\
    \quad Standard deviation & $0.250$ \\
    \midrule
    \multicolumn{2}{l}{\textit{Side-choice balance}} \\
    \quad Left / right & $48.2\%$ / $51.8\%$ \\
    \bottomrule
  \end{tabular}
\end{table}

\cref{tab:user-study} summarizes the results. Participants preferred CSFlow generations in $65.0\%$ of all judgments, with a $95\%$ confidence interval of $[60.0\%, 69.4\%]$ (Wilson score). Aggregating by participant rather than by judgment, $69.4\%$ of the participants preferred CSFlow generations in the majority of pairs they judged, $20.4\%$ preferred the baseline and $10.2\%$ were tied. The mean per-participant preference rate was $0.655$ (median $0.666$, standard deviation $0.250$), indicating that the effect is distributed across observers rather than produced by a small number of strongly opinionated ones. Side-choice rates were balanced at $48.2\%$ left against $51.8\%$ right, giving no evidence of positional bias.

\section{Frequency reconstruction analysis}
\label{app:spectral}

Flow and diffusion models often don't accurately model all aspects of the data distribution. The cartoonish-look discussed in~\cref{subsec:exp:analysis} can intuitively be associated with a weaker power of the mid-high frequencies. To quantify that, we analyze how well our method models the marginal frequency spectrum (visualized in Figure~\ref{fig:rapsd}) compared to the base model and inference.

We compare the radially averaged power spectral density (RAPSD, discussed in Appendix~\ref{rapsd-details}) of generated images against that of the data. For a method $m$ and radial frequency bin $f$ we report the log-space error of the marginal power spectrum,
\begin{align}
    \label{eq:rapsd-error}
    E_m(f)
    =
    \left|
        \log \mathrm{RAPSD}_{m}(f) - \log \mathrm{RAPSD}_{D}(f)
    \right| ,
\end{align}
where $\mathrm{RAPSD}_{m}$ is averaged over generations and $\mathrm{RAPSD}_{D}$ over the dataset. We work in log space, because natural image power spectra span several orders of magnitude across the bands (\cref{fig:rapsd}), so an absolute error would be dominated by the lowest frequencies.

Both analyses were conducted on ImageNet-256 using finetuned JiT-H with CFG. For the baseline we use the model finetuned for the same number of steps with a uniform inference grid; for ours we use the CSFlow-weighted JiT-H model that performed best in terms of FID. We used 5 runs to estimate the mean and standard deviation of the scores.

\begin{figure}[h]
    \centering
    \includegraphics[width=0.5\linewidth]{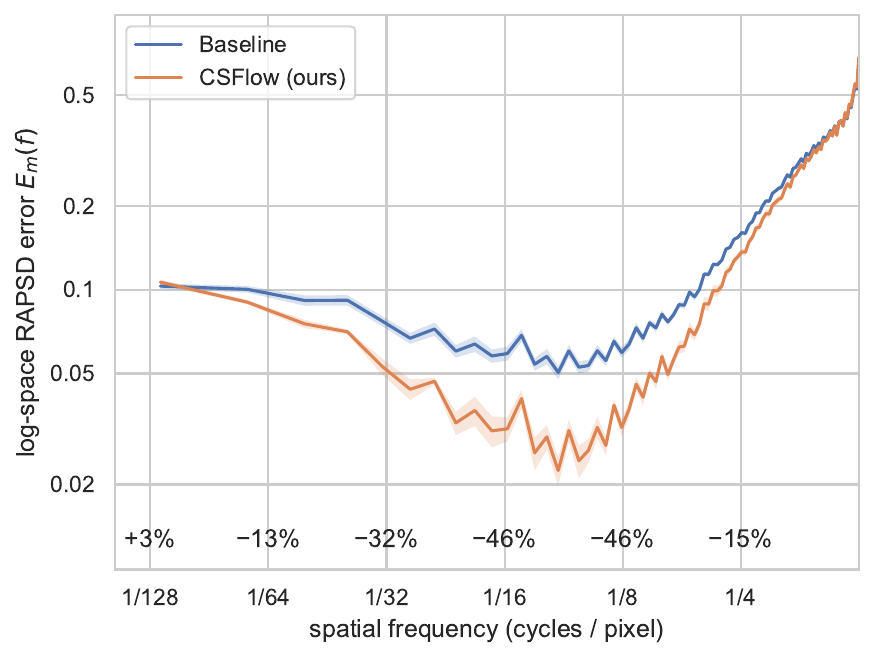}
    \caption{Log-space RAPSD error against spatial frequency for the baseline and CSFlow, with $\pm 1$ standard deviation bands over 5 runs. Our method shows improvements across almost all frequencies, with the highest gains in the medium freqiency range. Note that the fluctuations of the curve may be caused by the aggregation of the radial frequency bins in RAPSD.}
    \label{fig:rapsd-logerror}
\end{figure}

\cref{fig:rapsd-logerror} shows that CSFlow lowers the log-space spectrum error at every frequency, but the reduction is strongly non-uniform: it is largest in the mid-frequency band around $1/32$ to $1/8$ cycles per pixel, where the CSF retains between $47\%$ and $100\%$ of its peak value, and smallest at the two extremes, where the CSF is lower. In relative terms the reduction is $32.7\%$ to $51.0\%$ in the mid band against $1.0\%$ at the lowest and $6.2\%$ at the highest frequency. Overall, our method then produces improved modeling of the power spectrum, but the gains are higher in the middle-frequencies, where the CSF weight is high. 

\section{Sensitivity of the CSFlow weights}
\label{app:weight-sensitivity}

The CSF is defined over cycles per degree of visual angle, so converting it to the image frequency axis requires assumptions about how the image is viewed (\cref{app:csf-conversion}). This appendix examines how much the resulting weights depend on those assumptions. The short answer is that the three viewing parameters reduce to fewer degrees of freedom than they appear to have, and that the dependence is a shift along the log-frequency axis rather than a change in shape.

\subsection{Image size and viewing distance}

The weights depend on viewing distance, image size and image resolution. Of these, the first two are not independent: doubling the physical image size is exactly equivalent to halving the viewing distance, since both enter the conversion only through the visual angle subtended by one pixel. This is the usual scale ambiguity, and it means there is one parameter here, not two.

Changing that parameter rescales the arguments of the CSF multiplicatively. Because the CSF is conventionally plotted against log frequency, this corresponds to translating the frequency window along the CSF curve, leaving its shape intact. \cref{tab:weight-sensitivity} shows the effect: configurations 2 and 3 produce identical weights, as the equivalence predicts.

\subsection{Resolution}

Resolution behaves differently. Increasing the resolution while holding image size and viewing distance fixed raises the effective SNR of the low-frequency content that was already present, because coarse structure is now averaged over more pixels -- the observation that motivates the resolution-dependent schedule shifts of Simple Diffusion~\citep{sid} and Stable Diffusion 3~\citep{sd3}.

In our setting, however, a change in resolution cannot be reduced to such a timestep shift. Increasing the resolution also introduces frequencies that did not exist in the lower-resolution image, which extends the range of arguments passed to the CSF. The effect on the weights therefore depends on the CSF's behaviour in a region it was not previously evaluated on, and cannot be computed exactly without additional assumptions about the empirically estimated CSF. Configuration 4 in \cref{tab:weight-sensitivity} shows that halving the resolution gives weights close to, but distinct from, the size and distance changes.

\subsection{Other display parameters}

Display parameters that scale contrast uniformly across frequencies, such as brightness, have no effect on the weights. Our weighting depends only on the \emph{relative} sensitivity across frequencies, through the normalization in Eq.~\ref{eq:csflow_weight}, so a frequency-independent factor cancels.

\begin{table}[h]
  \caption{Mean $w_{\mathrm{CSFlow}}$ per interval of denoising time under four viewing configurations. Configuration 1 uses the assumptions of the main paper at $512 \times 512$. Configurations 2 and 3 are identical, as image size and viewing distance enter only through their ratio. Halving the resolution (4) is close to, but not the same as, either.}
  \label{tab:weight-sensitivity}
  \centering
  \small
  \setlength{\tabcolsep}{7pt}
  \renewcommand{\arraystretch}{1.05}
  \begin{tabular}{lccccc}
    \toprule
    Configuration & $[0.0, 0.2)$ & $[0.2, 0.4)$ & $[0.4, 0.6)$ & $[0.6, 0.8)$ & $[0.8, 1.0)$ \\
    \midrule
    1. Paper assumptions, $512 \times 512$ & 1.34 & 1.55 & 1.18 & 0.71 & 0.23 \\
    2. $2\times$ image size & 0.84 & 1.31 & 1.29 & 1.03 & 0.53 \\
    3. $\tfrac{1}{2}\times$ viewing distance & 0.84 & 1.31 & 1.29 & 1.03 & 0.53 \\
    4. $\tfrac{1}{2}\times$ resolution ($256$) & 0.91 & 1.36 & 1.28 & 0.97 & 0.49 \\
    \bottomrule
  \end{tabular}
\end{table}

\section{Proof of Theorem~\ref{thm:retained-signal}, Closed form of retained signal}
\label{apn:proof-retained-signal}

\begin{proof}
By linearity of the DFT,
\begin{align}
    F_f(\mathbf{x}_t)
    =
    a_tF_f(\mathbf{x}_{clean})
    +
    b_tF_f(\boldsymbol{\epsilon}).
\end{align}
Therefore,

\begin{align}
    \mathbb{E}[P_f(\mathbf{x}_t)]
    &=
    \mathbb{E}\left[
    \left|
        a_tF_f(\mathbf{x}_{clean})
        +
        b_tF_f(\boldsymbol{\epsilon})
    \right|^2
    \right] \\
    &=
    a_t^2
    \mathbb{E}\left[
        |F_f(\mathbf{x}_{clean})|^2
    \right]
    +
    b_t^2
    \mathbb{E}\left[
        |F_f(\boldsymbol{\epsilon})|^2
    \right] \\
    &\quad+
    2a_tb_t
    \operatorname{Re}
    \mathbb{E}\left[
        F_f(\mathbf{x}_{clean})
        \overline{F_f(\boldsymbol{\epsilon})}
    \right].
\end{align}

Since \(\boldsymbol{\epsilon}\) is independent of \(\mathbf{x}_{clean}\) and zero-mean, the cross term vanishes. 
By the flat unit expected power assumption (true for $\epsilon \sim \mathcal{N}(0,1)$),
\begin{align}
    \mathbb{E}\left[
        |F_f(\boldsymbol{\epsilon})|^2
    \right] = 1.
\end{align}
Thus,
\begin{align}
    \mathbb{E}[P_f(\mathbf{x}_t)]
    =
    a_t^2S_f + b_t^2.
\end{align}
Furthermore,
\begin{align}
    \mathbb{E}\left[P_f(a_t\mathbf{x}_{clean})\right]
    =
    a_t^2S_f.
\end{align}
Hence,
\begin{align}
    \frac{
        \mathbb{E}\left[P_f(a_t\mathbf{x}_{clean})\right]
    }{
        \mathbb{E}\left[P_f(\mathbf{x}_t)\right]
    }
    =
    \frac{
        a_t^2S_f
    }{
        a_t^2S_f + b_t^2
    }.
\end{align}
\end{proof}

\section{Extensions: model-dependent signal gain and equalized-gain objectives}
\label{app:extensions}

The CSFlow weights described in the \cref{method-chapter} are derived directly from the forward process. They depend only on the noise schedule $a_t, b_t$ and the dataset power spectrum $S_f$ via \cref{thm:retained-signal} and no model is evaluated to obtain them. The retained signal describes \emph{what information is present} in real images after noising them to some $t$ (rather than what could be reliably recovered using a denoising model). This is also why they do not transfer to latent models: in a latent flow there is no direct way to assess which pixel-space spectral information remains after noising to $t$, because a noisy latent is out of distribution for most decoders. As a result, analysing just the forward process could be misleading in case of latent models, as it may show that adding even small amounts of noise can remove all the information (since noisy latents are OOD), even though even a simple denoising model could reliably remove the noise and recover the information.

This appendix describes two future works that may follow from taking the denoiser into account.

\subsection{A model-dependent signal gain metric}
\label{app:diagnostic}

Instead of what survives noising, let's consider what a \emph{trained} model can restore. For a diffusion or flow model with parameters $\theta$, and for each image in a large evaluation set, we:
\begin{enumerate}[leftmargin=1.5em,itemsep=0pt,topsep=2pt]
    \item encode the image (omitted for pixel-space models),
    \item noise the representation to level $t$,
    \item denoise it back to $t=1$ using the flow/diffusion model,
    \item decode the result (omitted for pixel-space models).
\end{enumerate}
We can then take the RAPSD of the reconstruction error between the output and the original image to see the radial spectrum of error. This, normalized by the RAPSD of the original image itself, tells us which radial frequencies did the model not reconstruct accurately. Then we need to average it over the data and normalize between two reference points. The first reference is the error at $t=0$, where the representation is pure noise and nothing can be restored. The second, at $t=1$ is the error of the autoencoder itself, where no information has been lost to noise (for a pixel-space model this second reference is a zero vector). It is our lower bound of the information loss. The result, which we denote $\hat{r}^{\theta}_{\mathrm{signal}}(f,t)$, takes values in $[0,1]$ and indicates how much of the original signal at frequency $f$ the model can restore from noise level $t$.

This quantity is conceptually close to the retained signal $r_{\mathrm{signal}}(f,t)$ of \cref{thm:retained-signal}, but differs in one important part: it depends on a specific trained denoiser ($\theta$). The information does not need to  be present in $\mathbf{x}_t$ to be recovered, as long as the model can provide it from its learned prior. For example, let's consider a recognizable logo on a shirt. At a sufficiently high noise level the representation may no longer carry the logo's fine geometric details or textures, yet a model that has seen that specific logo multiple times can restore it correctly, since the coarse structure and surrounding context make it very clear which logo should it be. In such a case $r_{\mathrm{signal}}$ would report that the detail is destroyed while $\hat{r}^{\theta}_{\mathrm{signal}}$ would report that it is still recoverable.

$\hat{r}^{\theta}_{\mathrm{signal}}$ can be used in the same ways as $r_{\mathrm{signal}}$. It can be also visualized as in \cref{fig:r-signal-three-pngs} and its time derivative can measure the information gain per denoising step under that specific model, so it can be substituted for $\partial_t r_{\mathrm{signal}}$ in Eqs.~\ref{eq:raw_csf_score} and~\ref{eq:csflow_weight} to obtain model-specific weights. 
That can be useful in two ways.
First, in model analysis: the resulting map shows at which stage a given model commits to each level of detail. This can also be useful in comparing different latent spaces.
Second, in inference-time acceleration: if a band of highest frequencies is already fully restorable at a moderately high noise level, then there is no reason to do more than one step there.

Nevertheless, we expect some potential problems.
Most importantly, $\hat{r}^{\theta}_{\mathrm{signal}}$ is likely to show considerable overlap between the frequencies recovered at different noise levels, as many autoencoders distribute structural (low frequency) and detail (high frequency) information across channels with similar magnitude, and so both may end up being generated at the same time~\citep{dcae15}. 
Moreover, a spectral analysis of the final output may not reveal the non-spectral "decisions" that were made in early stages (eg. some semantic information about the species of a dog).
Furthermore, the $\hat{r}^{\theta}_{\mathrm{signal}}$ weights depend on and are precise only for a specific model, and thus if one would want to use such timestep weights during training, the timestep weights themselves would also need to be updated over time.

\subsection{Equal-weight latent space}
\label{app:equalized-gain}

Using \cref{sec:inverse-cdf} we build the timestep grid so that each step gets an equal CSFlow weight. A different approach would be to train a representation in which the weight is already equal across timesteps under a linear schedule.

To do that, one would need to repurpose the $\hat{r}^{\theta}_{\mathrm{signal}}$ introduced in the previous section. In this case, for computational feasibility (as this needs to happen during autoencoder training) we should consider working on individual-sample approximations of the $\hat{r}^{\theta}_{\mathrm{signal}}$. Explicitly, we can define the pixel-space reconstruction error map as:

\begin{align}
    z_1 &= E_{\theta}(x_{raw}) && \text{Encode the image}\\ 
    z_t &= z_1\cdot t + (1-t)\cdot \epsilon, \epsilon \sim \mathcal{N}(\mathbf{0},\mathbf{1}) && \text{Noise the latent}\\
    \hat{z}_1 &= G_{\phi}(z_t, t) && \text{Denoise the latent}\\
    \hat{x} &=D_\theta(\hat{z}) && \text{Decode the image}\\
    m_{x,t} &= x_{raw} - \hat{x} && \text{Get the pixel-space error map}
\end{align}
with $x_{raw}$ being the original pixel-space image, $E_\theta$ and $D_\theta$ defining the encoder and decoder, and $G_\phi$ the denoiser. We can then get the reconstruction error's spectral magnitude and it's radial average as:

\begin{align}
    m_{f,x, t} &= F_f(m_{x,t}) && \text{2D Fourier transform of the error}\\
    R_{x,t}(f) &= \frac{RAPSD_{m_{f,x,t}}(f)}{RAPSD_x(f)} && \text{vector of reconstruction error ratio} \\
    & && \text{for each radial frequency} \notag\\
    R_{CSF, x,t} &= \int_{\Omega}R_{x,t}(f) \cdot CSF(f) df &&\text{Get a scalar CSF-weighted }\\
    & && \text{frequency information loss/error} \notag
\end{align}

That allows us to get a single score for a specific image $x$ how much human-vision-relevant spectral information has been removed by noising it to a noise level $t$. Then enforcing linearity on $R_{CSF, x,t}$ w.r.t $t$ would ensure that each equal-sized denoising step of the latent diffusion/flow matching reduces the CSF-weighted spectral error by the same value.

We can enforce this linearity by enforcing the following constraint
\begin{align}
    \forall_{t,x,\epsilon}\frac{R_{CSF, x,t}-R_{CSF, x,1}}{R_{CSF, x,0}-R_{CSF, x,1}} = t
\end{align}

Satisfying that constraint is likely implausible, since it would require exactly linear information gain in the latent space, but it could be encouraged by the following regularization loss on $\theta$.

\begin{align}
    \mathcal{L}_{CSF\text{-linear}}&=
    \norm{ \frac{R_{CSF, x,t}-sg(R_{CSF, x,1})}{sg(R_{CSF, x,0})-sg(R_{CSF, x,1})} - t}_2\text{,}
\end{align}
with $sg$ denoting stopgrad. We believe it to be an interesting and promising future work direction that could enforce the equal perception-weighted information gain of each denoising step in latent-space diffusion and flow models. 

\newpage 
\section{More qualitative results}
More uncurated qualitative results can be viewed in Figure \ref{fig:more-results}.

\begin{figure}[h!]
    \centering
    \includegraphics[width=1\linewidth]{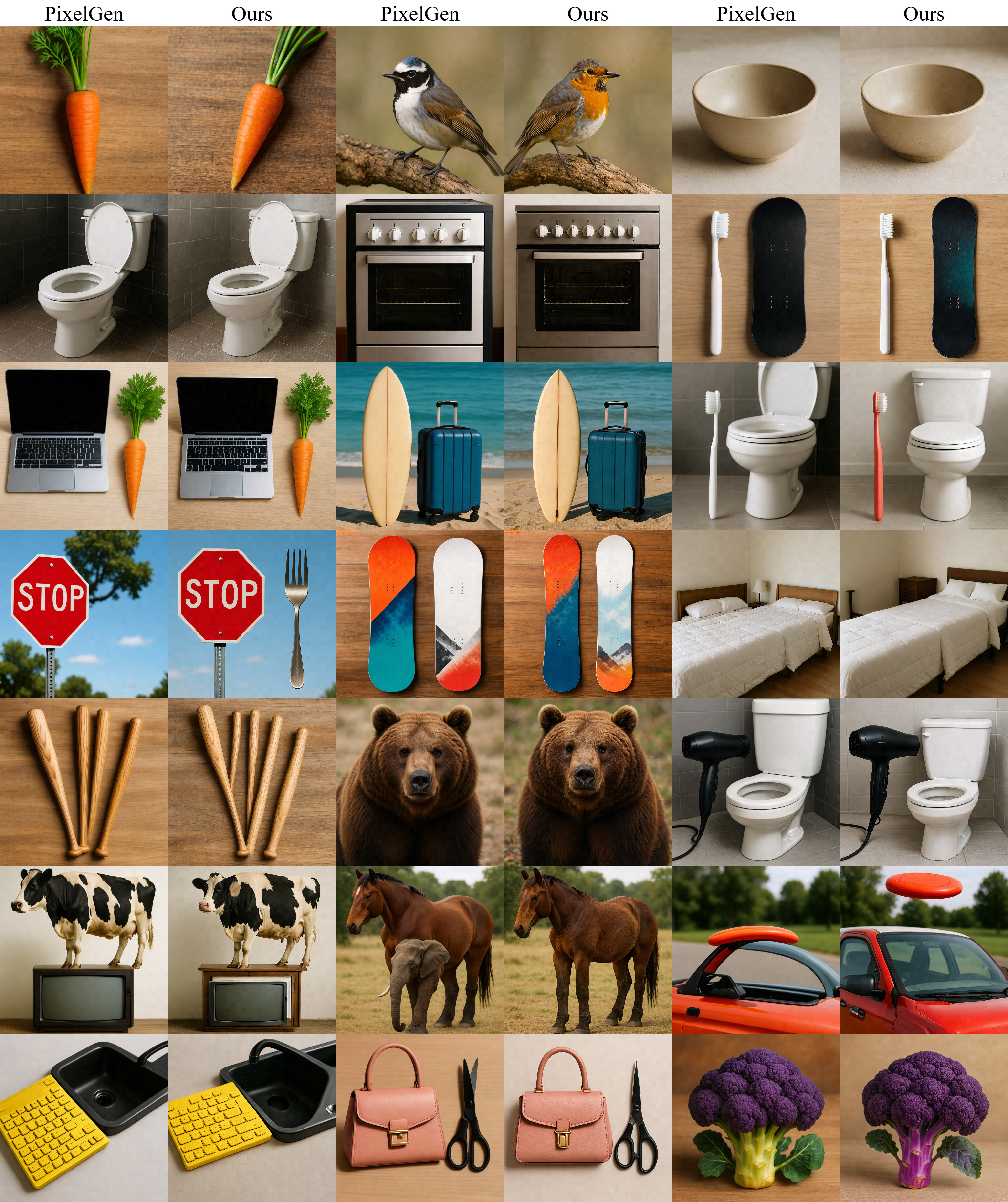}
    \caption{Pair comparisons between PixelGen-XXL/16 and CSFlow-weighted version (ours). Image pair indices were sampled randomly.}
    \label{fig:more-results}
\end{figure}


\end{document}